\documentclass[]{idlab_template}

\usepackage{marvosym} % 必须添加这个包才能识别 \Letter
\usepackage[toc,page,header]{appendix}
\usepackage{mathtools}
\mathtoolsset{showonlyrefs=true}
\usepackage{booktabs}
\usepackage{xcolor,colortbl}
\usepackage{siunitx} % 必须导入，用于小数点对齐
\usepackage{tabularx} % 必须引入，用于控制表格总宽度
\definecolor{darkgray}{gray}{0.1} % 定义比普通灰色更深的灰色

\definecolor{customgray}{HTML}{777777} 

\definecolor{headerblue}{RGB}{218, 232, 252} % 浅蓝色背景

\definecolor{posgreen}{RGB}{0, 100, 0} % 更深、更专业的墨绿色   
\definecolor{darkred}{RGB}{180,0,0}
\definecolor{headerpurple}{RGB}{235, 230, 250}
\definecolor{cSafe}{RGB}{235, 250, 235}   % 浅绿色 (Safe/Stable)
\definecolor{cRisk}{RGB}{255, 245, 235}   % 浅橙色 (Warning)
\definecolor{cBad}{RGB}{255, 230, 230}    % 浅红色 (Bad trend)
\definecolor{cDestruct}{RGB}{255, 215, 215} % 深一点的红色 (Destructive Highlighting)
\newcolumntype{Y}{>{\centering\arraybackslash}X}

\usepackage[T1]{fontenc}
\definecolor{idlabpurple}{HTML}{6F2DA8}
\hypersetup{
    colorlinks=true,       % 开启彩色链接
    linkcolor=idlabpurple,  % 内部引用 (图, 表, 公式)
    citecolor=idlabpurple,  % 文献引用 [1]
    urlcolor=idlabpurple,   % 网络链接 (http)
    allcolors=idlabpurple,  % 强制所有类型的链接统一颜色
    breaklinks=true         % 允许长链接换行
}
\AtBeginDocument{
    \hypersetup{citecolor=idlabpurple, linkcolor=idlabpurple, urlcolor=idlabpurple}
}
\usepackage{minitoc}
\usepackage{wrapfig}
\usepackage{cleveref} 
\usepackage{subcaption}
\usepackage{booktabs}
\usepackage{fontawesome5}
\usepackage{float}
\usepackage{natbib}
\usepackage{latexsym}

\usepackage{url}
\usepackage{amssymb}
\usepackage{amsmath}
\usepackage[utf8]{inputenc}
\usepackage{microtype}
\usepackage{booktabs}
\usepackage{pifont} 
\usepackage{multirow}
\usepackage{makecell}
\usepackage{paralist}
\usepackage{xspace}
\usepackage{color}
\usepackage{xcolor}
\usepackage{colortbl}
\usepackage{adjustbox}
\usepackage{hyperref} 
\usepackage[edges]{forest}
\usepackage{tikz} 
\usepackage{caption}
\usepackage{amsfonts}

\usepackage{longtable}

\definecolor{lightpurple}{RGB}{242, 235, 255}
\definecolor{linegray}{gray}{0.85}

\hypersetup{
    colorlinks,
    linkcolor={blue!80!black},
    citecolor={blue!80!black},
}
\tikzset{
    root/.style =             {align=center, text width=1cm, rounded corners=3pt, line width=0.3mm, fill=gray!10, draw=gray!80, font=\small},
    demographic/.style =         {align=center, text width=1.8cm, rounded corners=3pt, line width=0.3mm, fill=blue!10, draw=blue!80, font=\footnotesize},
    demographic_work/.style =    {align=center, text width=10cm, rounded corners=3pt, line width=0.3mm, fill=blue!10, draw=blue!0, font=\footnotesize},
    character/.style =         {align=center, text width=1.8cm, rounded corners=3pt, line width=0.3mm, fill=red!10, draw=red!80, font=\footnotesize},
    character_work/.style =    {align=center, text width=10cm, rounded corners=3pt, line width=0.3mm, fill=red!10, draw=red!0, font=\footnotesize},
    personalization/.style =           {align=center, text width=1.8cm, rounded corners=3pt, line width=0.3mm, fill=cyan!10, draw=cyan!80, font=\footnotesize},
    personalization_work/.style =      {align=center, text width=10cm, rounded corners=3pt, line width=0.3mm, fill=cyan!10, draw=cyan!0, font=\footnotesize},
    risk/.style =         {align=center, text width=1.8cm, rounded corners=3pt, line width=0.3mm, fill=orange!10, draw=orange!80, font=\footnotesize},
    risk_work/.style =    {align=center, text width=10cm, rounded corners=3pt, line width=0.3mm, fill=orange!10, draw=orange!0, font=\footnotesize},
}

\usepackage{CJK}
\usepackage[most]{tcolorbox}
\usepackage{listings}

\newcommand{\masktok}[1]{%
    \colorbox{red!12}{%
        \textcolor{red!70!black}{\textbf{\texttt{#1}}}%
    }%
}

\lstdefinestyle{maskedcode}{
    basicstyle=\ttfamily\footnotesize,
    columns=fullflexible,
    keepspaces=true,
    breaklines=true,
    showstringspaces=false,
    frame=none,
    aboveskip=3pt,
    belowskip=3pt,
    escapeinside={(*@}{@*)}
}
\definecolor{stapoMathViolet}{HTML}{654BA0}
\definecolor{stapoMathRust}{HTML}{B4602B}
\definecolor{stapoMathTeal}{HTML}{197E84}
\definecolor{stapoMathBorder}{HTML}{CDD0D9}
\definecolor{stapoMathMaskBg}{HTML}{FFE3E5}
\definecolor{stapoMathMaskFg}{HTML}{A52F3A}
\colorlet{stapoMathAccent}{stapoMathViolet}

\providecommand{\stapoMathKeep}[1]{%
  \par
  \ifdim\pagegoal<\maxdimen
    \ifdim\dimexpr\pagegoal-\pagetotal\relax<#1\relax
      \newpage
    \fi
  \fi
}

\providecommand{\stapoMathMaskText}[1]{%
  \begingroup
  \setlength{\fboxsep}{1.2pt}%
  \colorbox{stapoMathMaskBg}{\textcolor{stapoMathMaskFg}{#1}}%
  \endgroup
}
\providecommand{\stapoMathMask}[1]{%
  \ifmmode
    \text{\stapoMathMaskText{$#1$}}%
  \else
    \stapoMathMaskText{\strut #1}%
  \fi
}

\providecommand{\stapoMathCategory}[4]{%
  \par\stapoMathKeep{12\baselineskip}%
  \colorlet{stapoMathAccent}{#3}%
  \refstepcounter{subsubsection}% labels immediately after the banner work
  \addcontentsline{toc}{subsubsection}{Category #1: #2}%
  \par\vspace{7pt}%
  \begin{tcolorbox}[
    enhanced,
    colback=stapoMathAccent!9,
    colframe=stapoMathAccent!28,
    boxrule=0.5pt,
    arc=2.3mm, outer arc=2.3mm,
    left=3mm,right=3mm,top=2mm,bottom=2mm,
    fontupper=\normalsize,
    before skip=0pt,after skip=5pt]
    \noindent{\bfseries\textcolor{stapoMathAccent}{Category #1: #2}}%
    \hfill{\normalfont\textcolor{black!65}{#4 cases}}%
  \end{tcolorbox}%
}

\ifcsname stapoMathCase\endcsname\else
  \newtcolorbox{stapoMathCase}[1]{%
    enhanced,
    breakable,
    colback=white,
    colframe=stapoMathBorder,
    colbacktitle=stapoMathAccent!8,
    coltitle=black,
    fonttitle=\bfseries\normalsize,
    fontupper=\normalsize,
    title={#1},
    boxrule=0.6pt,
    arc=2.5mm,outer arc=2.5mm,
    left=3mm,right=3mm,top=1.8mm,bottom=2mm,
    toptitle=1.3mm,bottomtitle=1.3mm,
    before skip=5pt,after skip=8pt,
    parbox=false
  }
\fi

\providecommand{\stapoMathStat}[4]{%
  \par\vspace{1mm}%
  {\color{black!18}\noindent\rule{\linewidth}{0.4pt}}\par\vspace{0.5mm}%
  \noindent\textbf{Masked token:}\enspace #1\qquad
  \textbf{P:}\enspace #2\qquad
  \textbf{Adv.:}\enspace #3\par
  {\raggedright\noindent\textbf{Top-5:}\enspace #4\par}%
}

\newtcolorbox{codecase}[1]{
    enhanced,
    breakable,
    colback=gray!3,
    colframe=gray!40,
    colbacktitle=blue!7,
    coltitle=black,
    title={#1},
    fonttitle=\small\bfseries,
    boxrule=0.5pt,
    arc=1.5mm,
    left=2mm,
    right=2mm,
    top=1.5mm,
    bottom=1.5mm,
    before skip=6pt,
    after skip=6pt
}
\definecolor{reasoncolor}{HTML}{B45B58}
\definecolor{rewardcolor}{HTML}{368A78}

\newtcolorbox{casenote}[2]{
    enhanced,
    colback=#1!5!white,
    frame hidden,
    borderline west={2pt}{0pt}{#1},
    sharp corners,
    boxsep=0pt,
    left=7pt,
    right=6pt,
    top=5pt,
    bottom=5pt,
    before skip=5pt,
    after skip=5pt,
    fontupper=\small,
    before upper={
        {\sffamily\bfseries\footnotesize
        \textcolor{#1}{#2}}
        \par\smallskip
    }
}

\usepackage{amsthm}
\usepackage{amsmath}   % 数学公式增强

\usepackage{algorithm}
\usepackage{algpseudocode}
\usepackage{amssymb}  
\usepackage{graphicx}
\newtheorem{theorem}{Theorem}[section]      % 定理，按 section 编号
\newtheorem{lemma}[theorem]{Lemma}          % 引理，共用 theorem 编号

\theoremstyle{definition}
\newtheorem{definition}[theorem]{Definition}

\definecolor{lightgraybg}{gray}{0.95}
\definecolor{textgray}{gray}{0.55} % 稍微深一点的灰，保证打印可见性

\definecolor{purplishgray}{HTML}{F0F8FF}

\definecolor{textgray}{gray}{0.4}          
\definecolor{darkgrayval}{gray}{0.35}
\newcommand{\tabscore}[2]{%
  \makebox[2.4em][c]{#1}%
  \hspace{0.4em}%
  \makebox[2.9em][c]{\footnotesize\color{textgray}(#2)}%
}

\makeatletter 
\makeatletter
\renewcommand\fs@ruled{%
  \def\@fs@cfont{\bfseries}\let\@fs@capt\floatc@ruled
  \def\@fs@pre{\color{black}\hrule height.8pt depth0pt \kern2pt}% 强制顶线黑色
  \def\@fs@post{\color{black}\kern2pt\hrule\relax}%             强制底线黑色
  \def\@fs@mid{\color{black}\kern2pt\hrule\kern2pt}%            强制中线黑色
  \let\@fs@iftopcapt\iftrue}
\makeatother

\title{STAPO: Stabilizing Reinforcement Learning for LLMs by Silencing Rare Spurious Tokens}
\author[1,2, \textdagger]{Shiqi Liu}
\author[1,2, \textdagger]{Zeyu He}
\author[1,2, \textdagger]{Guojian Zhan}
\author[1,2]{Letian Tao}
\author[1,2]{Zhilong Zheng}
\author[1]{Jiang Wu}
\author[1,2]{Yinuo Wang}
\author[1]{\\ Yang Guan}
\author[2]{Kehua Sheng}
\author[2]{Bo Zhang}
\author[1]{Keqiang Li}
\author[1,2,\Letter]{Jingliang Duan}
\author[1,\Letter]{Shengbo Eben Li}

\affiliation[1]{School of Vehicle and Mobility \& College of AI, Tsinghua University}
\affiliation[2]{Didi Voyager Labs, DiDi Autonomous Driving}
\contribution[\dagger]{Equal contribution}
\contribution[\mbox{\textnormal{\Letter}}]{Corresponding author}

\abstract{
Reinforcement Learning (RL) has significantly improved large language model reasoning, but existing RL fine-tuning methods rely heavily on heuristic techniques such as entropy regularization and reweighting to maintain stability. In practice, they often suffer from late-stage performance collapse, leading to degraded reasoning quality and unstable training.
    We identify a key factor behind this instability: a small fraction of tokens, termed spurious tokens (around 0.01\%), which contribute little to the reasoning outcome but receive disproportionately amplified gradient updates due to inheriting the full sequence-level reward.
    We present a unified framework for evaluating token-level optimization impacts across spurious risk, gradient norms, and entropy changes.
    Building on the analysis of token characteristics that severely disrupt optimization, we propose the Silencing Spurious Tokens (S2T) mechanism to efficiently suppress their gradient perturbations.
    Incorporating this mechanism into a group-based objective, we propose Spurious-Token-Aware Policy Optimization (STAPO), which promotes stable and effective large-scale model refinement.
    Across six mathematical benchmarks and three model scales, STAPO consistently demonstrates superior entropy stability and achieves average accuracy improvements over baselines including GRPO, 20-Entropy, and JustRL under two widely-adopted evaluation settings.

}

\date{February 17, 2026}
\correspondence{
S. E. Li and J. Duan with email \email{lishbo@tsinghua.edu.cn}.
}
\begin{document}

\maketitle

\vspace{-2em}

% \begin{figure}[htbp]
%     \centering
%     \hspace*{-1cm}\includegraphics[width=0.9\linewidth]{figures/intro.pdf}
%     \vspace{-2mm}
% \caption{\textbf{Core Idea.}
% (a) Conceptual analogy: We argue that spurious tokens, which are rare and uninformative tokens within otherwise correct responses that receive disproportionately large gradient updates, can harm training stability, analogous to a dissonant vocalist disrupting the harmony of a performance.
% (b) By masking this negligible fraction (near $ 0.01\%$) of spurious tokens {during the RL process of} Qwen3-8B-Base, STAPO approaches the Pareto frontier of performance ({AIME24} Acc) and entropy stability, compared to GRPO, 20-Entropy, and JustRL.}

%     \vspace{-6mm}\label{fig:intro}
% \end{figure}
%不需要目录就注释掉 注意目录不要和第一页放在一块 要有\newpage
%\newpage
%\tableofcontents
%\newpage

% \vspace{-2em}
% \noindent\makebox[\textwidth]{
% \parbox{0.8\textwidth}{\centering
% \textit{``All that glisters is not gold.''}\\[0.5em]
% \hfill --- William Shakespeare, \textit{The Merchant of Venice}
% }}

% \newpage
\begin{figure}[htbp]
    \centering
        \vspace{-4mm}
    \hspace*{-1cm}\includegraphics[width=0.9\linewidth]{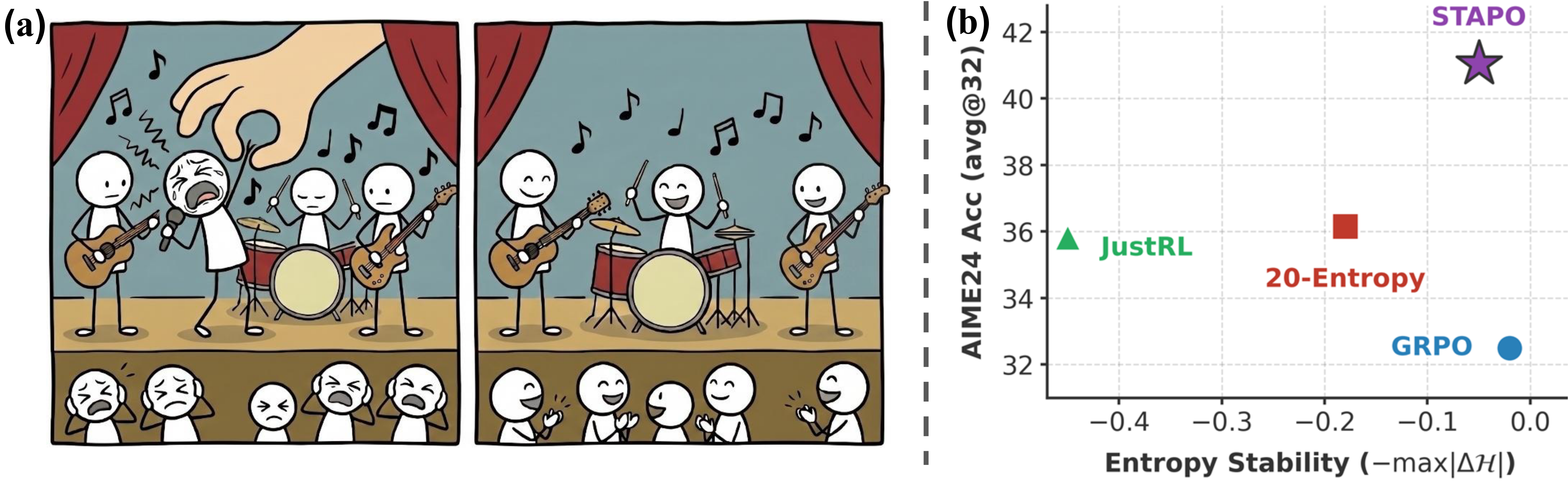}
    \vspace{-2mm}
\caption{\textbf{Core Idea.}
(a) In RLVR training, we argue that spurious tokens, which are rare and uninformative tokens within otherwise correct responses, can harm training stability, analogous to a dissonant vocalist disrupting the harmony of a performance.
(b) By masking this negligible fraction (near $ 0.01\%$) of spurious tokens during the RL process of Qwen3-8B-Base, STAPO approaches the Pareto frontier of AIME24 Acc and entropy stability, compared to GRPO, 20-Entropy, and JustRL.}

    \vspace{-4mm}\label{fig:intro}
\end{figure}

\section{Introduction}
\label{Section:intro}
Recent large language models (LLMs), including OpenAI-o1~\cite{gpt4}, DeepSeek-R1~\cite{guo2025deepseek}, and Qwen3~\cite{yang2025qwen3}, have demonstrated remarkable capabilities in complex reasoning domains such as mathematics and programming. Central to this success is Reinforcement Learning (RL)~\cite{li2023reinforcement}, which optimizes directly for outcome-level correctness and has been empirically linked to the emergence of advanced, long-horizon reasoning behaviors~\cite{cao2024survey}.

Existing approaches have largely focused on enhancing exploration to improve the learning capacity of LLMs under RL training. For example, DAPO~\cite{yu2025dapo} increases the clipping threshold, while SAPO~\cite{gao2025soft} adopts a soft-clipping strategy. Other works instead emphasize the exploration of forking tokens, as introduced in 20-Entropy~\cite{wang2025beyond}. However, in practice, these approaches often drive entropy toward explosive growth, ultimately leading to catastrophic degradation, where models collapse from coherent reasoning into shallow, repetitive, or even nonsensical outputs~\cite{zhang2025survey}.

To mitigate such entropy-induced instability, prior efforts have introduced various stabilization techniques, including entropy regularization~\cite{cui2025entropy,yang2025entropic}, sample augmentation~\cite{simoni2025gtpo,qiu2025noisygrpo}, and advantage reweighting~\cite{yang2025not}. These methods regulate entropy at a global level but remain coarse-grained and fail to capture the heterogeneous roles of individual tokens during optimization. As a result, they may either over-suppress useful exploration or induce oscillatory entropy dynamics, leading to suboptimal performance.

In this work, we aim to strike a balance between effective exploration and strong reasoning capabilities. We first introduce the concept of \emph{spurious tokens}, defined as a sparse subset of tokens within otherwise accurate responses that, rather than contributing to the underlying reasoning process, introduce logical misdirection or harmful noise, as illustrated in Figure~\ref{fig:intro}. From a token-level perspective, we systematically analyze the detrimental impact of these spurious tokens on optimization dynamics, demonstrating that they introduce misleading and destabilizing update signals. Building on this insight, we propose an efficient method for identifying spurious tokens and formulate Spurious-Token-Aware Policy Optimization (STAPO). By masking a negligible fraction (approximately $0.01\%$) of such tokens during training, STAPO significantly stabilizes policy entropy and consistently yields performance improvements. Overall, our main contributions are summarized as follows:

\begin{itemize}
\item We identify spurious tokens as a key source of training instability: a minor fraction of tokens (around $0.01\%$) that contribute little to the actual reasoning yet receive disproportionately large gradient updates by inheriting the full sequence-level reward.
To analyze this issue, we present a unified framework that systematically evaluates token-level optimization dynamics in terms of spurious risk, gradient norms, and entropy changes.

\item By analyzing token characteristics that severely disrupt optimization, we propose the Silencing Spurious Tokens (S2T) mechanism to efficiently identify such disruptive tokens and suppress their gradient perturbations. We further integrate S2T into a group-based objective and develop STAPO for stable and effective large-scale model refinement.

\item 
Across six mathematical reasoning benchmarks (AIME24, AIME25, AMC23, MATH500, Minerva, and OlympiadBench) and three model scales (Qwen3 1.7B, 8B, and 14B base models~\cite{yang2025qwen3}), STAPO substantially stabilizes policy entropy and consistently improves reasoning accuracy under diverse evaluation settings.

\end{itemize}

\section{Preliminaries}

\subsection{Problem Formulation}\label{sec:problem_for}
We consider the problem of fine-tuning large language models (LLMs) via reinforcement learning (RL) for reasoning tasks. Let $\mathcal{D}$ denote a distribution over input prompts. Given a prompt $\bm{x} \sim \mathcal{D}$, an LLM parameterized by $\theta$ acts as a stochastic policy $\pi_\theta$ that autoregressively generates an output sequence $\bm{y} = (y_1, \dots, y_T)$. Specifically, at each step $t$, a token $y_t \in \mathcal{V}$ is sampled according to $\pi_\theta(y_t \mid \bm{x}, \bm{y}_{<t})$.

Supervision is provided via a sparse, sequence-level verifiable reward $R(\bm{x}, \bm{y}) \in \{-1, 1\}$, which evaluates the correctness of the generated sequence $\bm{y}$ using an external verifier (e.g., a code compiler or a mathematical rule checker). The optimization objective is to maximize the expected reward:
\begin{equation}
\mathcal{J}(\theta) = \mathbb{E}_{\bm{x} \sim \mathcal{D},\, \bm{y} \sim \pi_\theta(\cdot \mid \bm{x})} \left[ R(\bm{x}, \bm{y}) \right].
\end{equation}

\subsection{Group Relative Policy Optimization (GRPO)}
We briefly review Group Relative Policy Optimization (GRPO)~\cite{shao2024deepseekmath}, which estimates advantages without relying on an explicit value function. For each prompt $\bm{x}$, GRPO samples a group of $G$ output sequences $\{\bm{y}_1, \dots, \bm{y}_G\}$ from a reference behavior policy $\pi_{\theta_{\mathrm{old}}}$. The optimization objective is defined as the average clipped surrogate loss over the sampled group:
\begin{equation}\label{eq:grpo_loss}
\begin{aligned}
\mathcal{J}_{\text{GRPO}}(\theta) &= \mathbb{E}_{\bm{x}\sim \mathcal{D}, \{\bm{y}_i\}_{i=1}^G\sim\pi_{\theta_{\mathrm{old}}}(\cdot \mid \bm{x})} \Bigg[ \frac{1}{G} \sum_{i=1}^{G} \frac{1}{|\bm{y}_i|} \sum_{t=1}^{|\bm{y}_i|} \min \biggl( \rho_{i,t}(\theta) \hat{A}_i, \\
&\quad \mathrm{clip}\bigl(\rho_{i,t}(\theta), 1-\epsilon, 1+\epsilon\bigr)\hat{A}_i \biggr) \Bigg] -\beta \mathbb{D}_\mathrm{KL}(\pi_\theta \mid \pi_\mathrm{ref}),
\end{aligned}
\end{equation}
\begin{equation}\label{eq:ratio_adv_def}
\begin{aligned}
    \rho_{i,t}(\theta) &= \frac{\pi_\theta(y_{i,t} \mid \bm{x}, \bm{y}_{i,<t})}{\pi_{\theta_{\mathrm{old}}}(y_{i,t} \mid \bm{x}, \bm{y}_{i,<t})}, \quad
    \hat{A}_i = \frac{R(\bm{x}, \bm{y}_i) - \operatorname{mean}(\{R(\bm{x}, \bm{y}_j)\}_{j=1}^G)}{\operatorname{std}(\{R(\bm{x}, \bm{y}_j)\}_{j=1}^G)},
\end{aligned}
\end{equation}
where $\rho_{i,t}(\theta)$ is the importance sampling ratio and $\hat{A}_i$ is the sequence-level advantage signal, derived by standardizing the reward across the $G$ samples within the group.
The $\pi_\mathrm{ref}$ serves as the reference policy, and $\beta$ is the scaling coefficient.

\subsection{Clip-Higher and Token Normalization}
Building upon the GRPO framework, several large-scale RL algorithms have recently emerged. Notably, DAPO~\cite{yu2025dapo} removes the KL penalty and introduces a set of training enhancements, with token-level normalization and an asymmetric clip-higher mechanism proving particularly effective for stabilizing optimization. Subsequent work, such as JustRL~\cite{he2025justrl}, has adopted these two components and demonstrated strong empirical performance.
The corresponding objective is written as:
\begin{equation}\label{eq:dapo_loss}
\begin{aligned}
\mathcal{J}_{\text{DAPO}}(\theta) &= \mathbb{E}_{\bm{x}\sim \mathcal{D}, \{\bm{y}_i\}_{i=1}^G\sim\pi_{\theta_{\mathrm{old}}}(\cdot \mid \bm{x})} \Bigg[ \frac{1}{\sum_{i=1}^G |\bm{y}_i|} \sum_{i=1}^{G} \sum_{t=1}^{|\bm{y}_i|} \min \biggl( \rho_{i,t}(\theta) \hat{A}_{i}, \\
&\quad \mathrm{clip}\bigl(\rho_{i,t}(\theta), 1-\epsilon_\mathrm{low}, 1+\epsilon_\mathrm{high}\bigr) \hat{A}_{i} \biggr) \Bigg],
\end{aligned}
\end{equation}
where $\epsilon_\mathrm{low}$ and $\epsilon_\mathrm{high}$ denote the asymmetric clipping parameters integral to the clip-higher mechanism. Given its simplicity and robust empirical success, we adopt this augmented configuration as our baseline objective.

\section{Methodology}
\label{sec:method}
\subsection{The Hidden Threat: Spurious Tokens}\label{subsec:spurious_def}

As discussed in Section~\ref{sec:problem_for}, rewards in Reinforcement Learning with Verifiable Rewards (RLVR) are typically derived solely from the final outcome. Consequently, all tokens $y_{i,t}$ within a given trajectory share an identical sequence-level advantage, $\hat{A}_i$. This coarse-grained credit assignment can inadvertently reinforce extraneous tokens. We formalize this phenomenon as follows:

\begin{definition}[Spurious Tokens]
  Spurious tokens are intermediate tokens $y_{i,t}$ that contribute negligibly to the correct reasoning outcome, yet receive disproportionately large positive updates.
\end{definition}

To empirically examine this effect, we train a Qwen3-1.7B base model under the JustRL setting on DAPO-MATH-17K~\cite{yu2025dapo}, and record all generated tokens along with their associated statistics during training. Figure~\ref{fig:spurious_examples} presents concrete instances of this phenomenon, showing that spurious tokens can induce misleading update signals that steer the policy toward detrimental directions and destabilize training.

Motivated by this observation, we propose Spurious-Token-Aware Policy Optimization (STAPO), which introduces a binary mask, $\mathbb{I}^{\mathrm{Spurious}}_{i,t}$, to discard gradient contributions from spurious tokens:
\begin{equation}\label{eq:spurious_mask}
  \mathbb{I}^{\mathrm{Spurious}}_{i,t} =
  \begin{cases}
    0, & \text{if } y_{i,t} \in \mathcal{S}_i, \\
    1, & \text{otherwise},
  \end{cases}
\end{equation}
where $\mathcal{S}_i$ denotes the set of identified spurious tokens within the $i$-th trajectory. Incorporating this, the STAPO objective is formulated as:
\begin{equation}\label{eq:stapo_loss}
  \begin{aligned}
    \mathcal{J}_{\mathrm{STAPO}}(\theta) & = \mathbb{E}_{\bm{x}\sim \mathcal{D}, \{\bm{y}_i\}_{i=1}^G\sim\pi_{\theta_{\mathrm{old}}}(\cdot \mid \bm{x})} \Bigg[ \frac{1}{\sum_{i=1}^{G} \sum_{t=1}^{|\bm{y}_i|} \mathbb{I}^{\mathrm{Spurious}}_{i,t}} \sum_{i=1}^{G} \sum_{t=1}^{|\bm{y}_i|} \mathbb{I}^{\mathrm{Spurious}}_{i,t} \cdot \min \biggl( \rho_{i,t}(\theta) \hat{A}_{i}, \\
                                         & \quad \mathrm{clip}\bigl(\rho_{i,t}(\theta), 1-\epsilon_\mathrm{low}, 1+\epsilon_\mathrm{high}\bigr) \hat{A}_{i} \biggr) \Bigg].
  \end{aligned}
\end{equation}

The terms $\rho_{i,t}$ and $\hat{A}_i$ follow standard definitions provided in \eqref{eq:ratio_adv_def}. Comparing the STAPO objective in Eq.~\eqref{eq:stapo_loss} with the standard DAPO objective in Eq.~\eqref{eq:dapo_loss}, two primary distinctions emerge. First, STAPO leverages this binary mask to selectively zero out the loss calculations for spurious tokens. Second, the normalization term in Eq.~\eqref{eq:stapo_loss} is dynamically adjusted to average the loss exclusively over the remaining valid tokens.

\begin{figure*}[t!]
  \centering

  % =======================
  % 左侧列：包含一个单独的大图 (原右侧图)
  % =======================
  \begin{minipage}[b]{0.605\linewidth}

    % --- 子图 (a): Spurious Tokens 示例 ---
    \begin{subfigure}[b]{\linewidth}
      \centering
      \includegraphics[width=\linewidth]{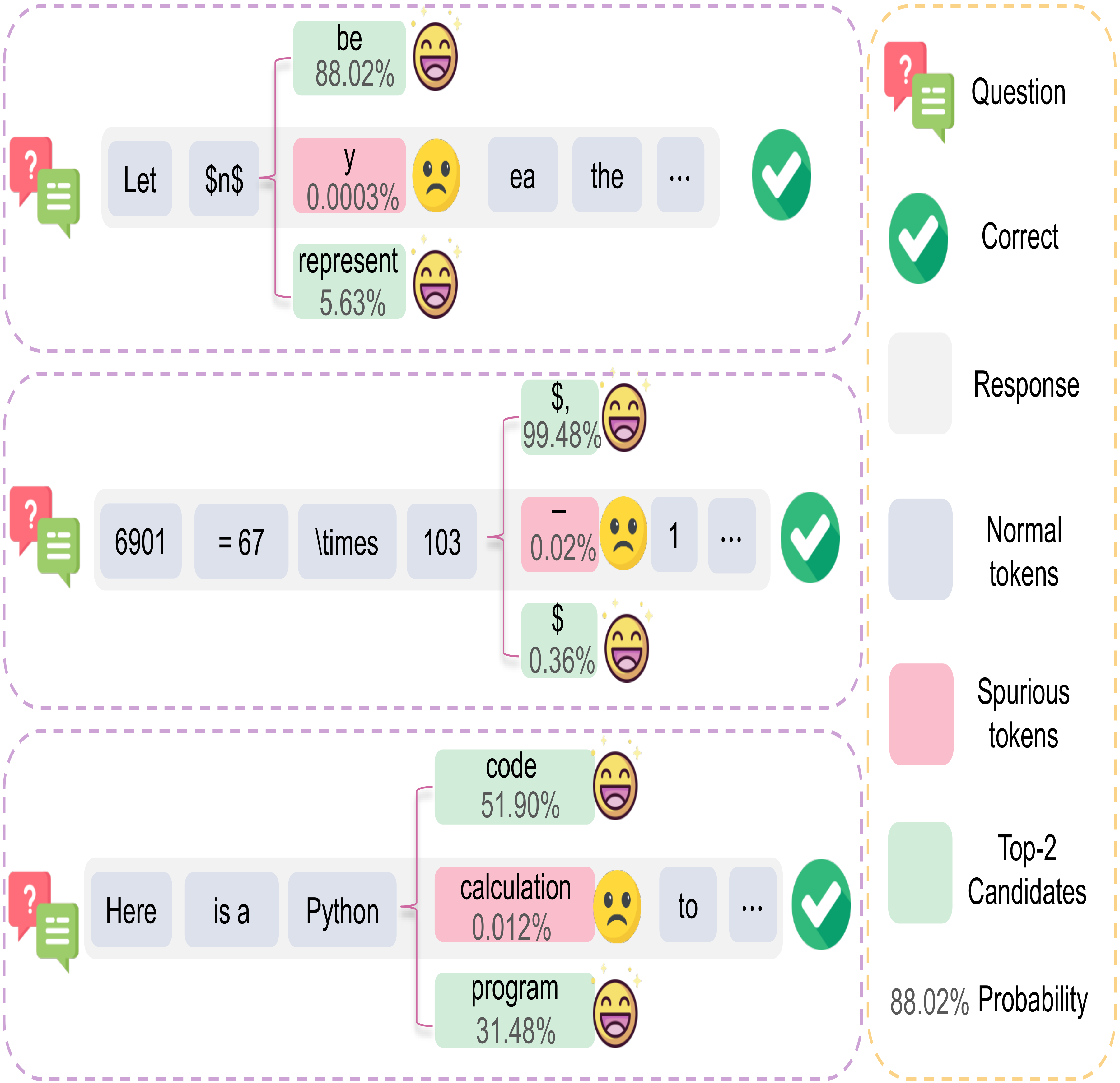}
      \caption{{Illustrative examples of spurious tokens.} Three instances where the LLM selects spurious tokens (pink) over reasonable top-2 candidates (green). Spurious tokens often represent semantic errors despite the final response's correctness. More comprehensive examples are provided in Appendix~\ref{app:spurious_examples}.}
      \label{fig:spurious_examples}
    \end{subfigure}

  \end{minipage}
  \hfill
  % =======================
  % 右侧列：包含两个竖排的子图 (原左侧图)
  % =======================
  \begin{minipage}[b]{0.385\linewidth}

    % --- 子图 (b): 散点图 (Token Quadrants) ---
    \begin{subfigure}[b]{\linewidth}
      \centering
      \includegraphics[width=\linewidth]{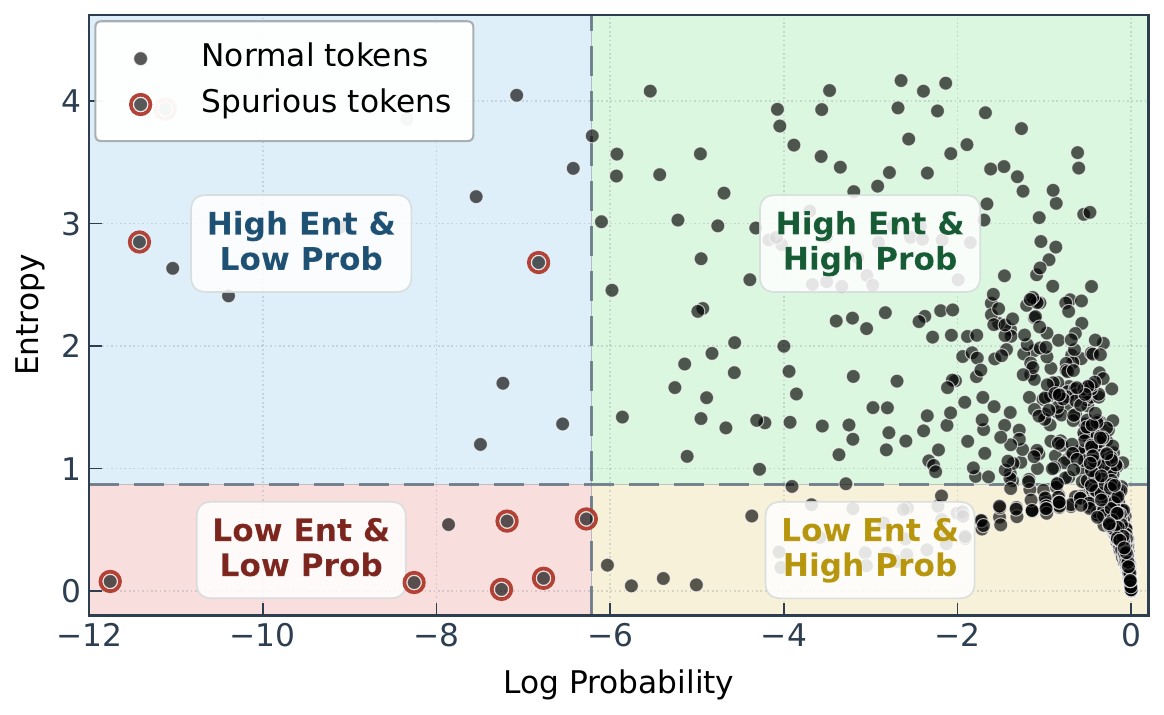}
      \vspace{-6mm}
      \caption{Token quadrants visualization}
      \label{fig:quadrants}
    \end{subfigure}

    % --- 子图 (c): 柱状图 (Mean Gradient) ---
    \begin{subfigure}[b]{\linewidth}
      \centering
      \includegraphics[width=\linewidth]{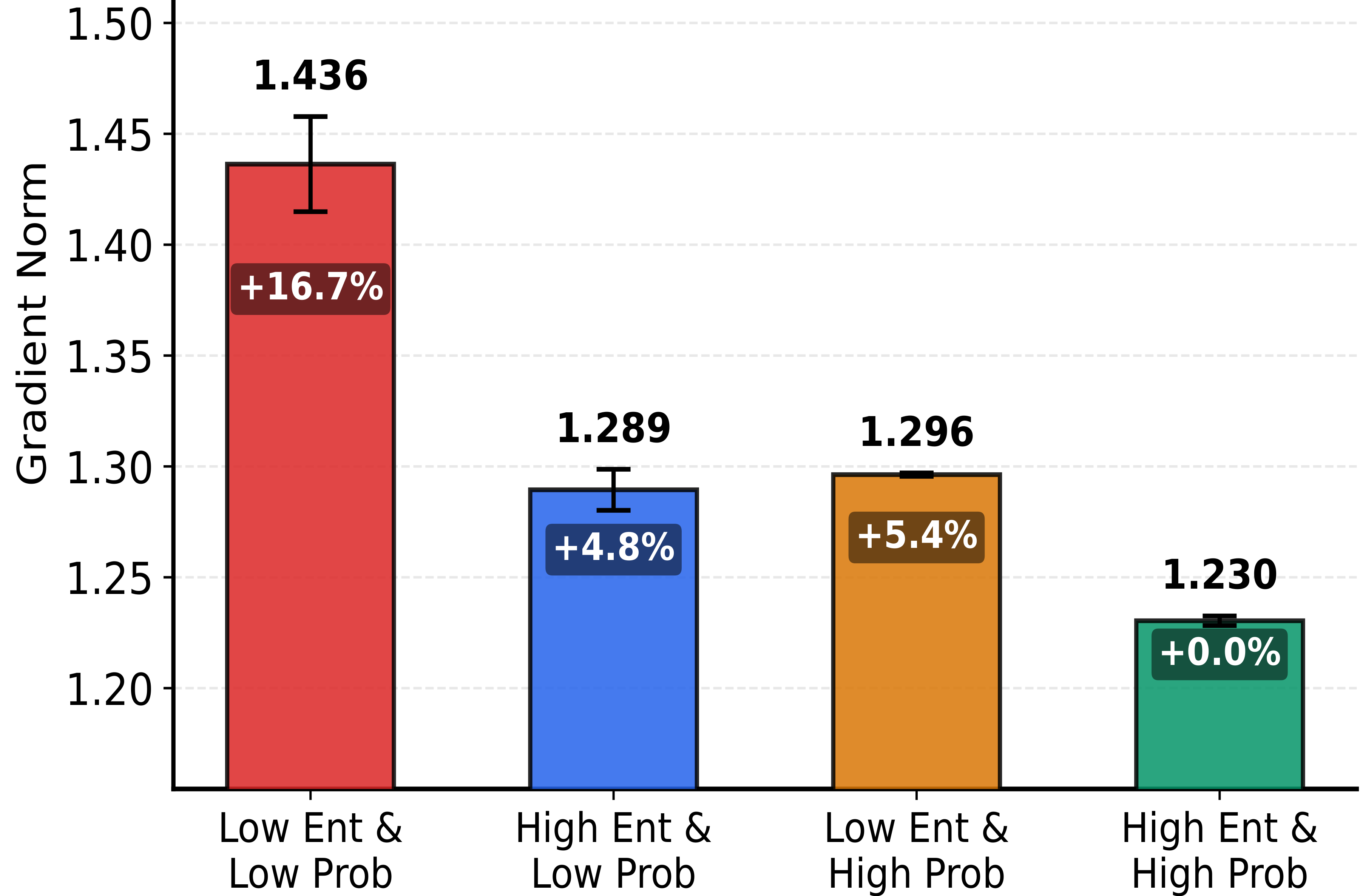}
      \vspace{-6mm}
      \caption{Gradient norm comparison}
      \label{fig:mean_grad}
    \end{subfigure}

  \end{minipage}
  \vspace{-4mm}
  \caption{\textbf{Comprehensive Analysis of Spurious Tokens.}}
  \label{fig:combined_analysis}
  % \vspace{-4mm}
\end{figure*}

\subsection{Token-Level Optimization Analysis}
\label{subsec:eg_coupling_analysis}

% % {r} 表示图片靠右 (Right)
% % {0.45\textwidth} 是图片占据的宽度，您可以根据需要调整

Fundamentally, the generation dynamics of LLMs are shaped by pre-training, which assigns lower prior probabilities to ill-formed or semantically anomalous token sequences.
As a result, spurious tokens, often appearing as incoherent or illogical reasoning steps, tend to exhibit \emph{low sampling probabilities}. Motivated by this, we extend JustRL with full masking of low-probability tokens within correct answers (JustRL-FullMask) and compare its entropy dynamics with the default one. As shown in Figure~\ref{fig:entropy_aime24}, JustRL-FullMask leads to severe entropy collapse, resulting in insufficient exploration and degraded performance, while the default JustRL update maintains persistently high policy entropy in later stages, hindering convergence. This trade-off exposes a paradox in entropy control, explained by the following update dynamics:

\begin{wrapfigure}{r}{0.50\textwidth}
  \vspace{-0mm} % 微调图片顶部的多余留白（视情况增减）
  \centering
  \includegraphics[width=\linewidth]{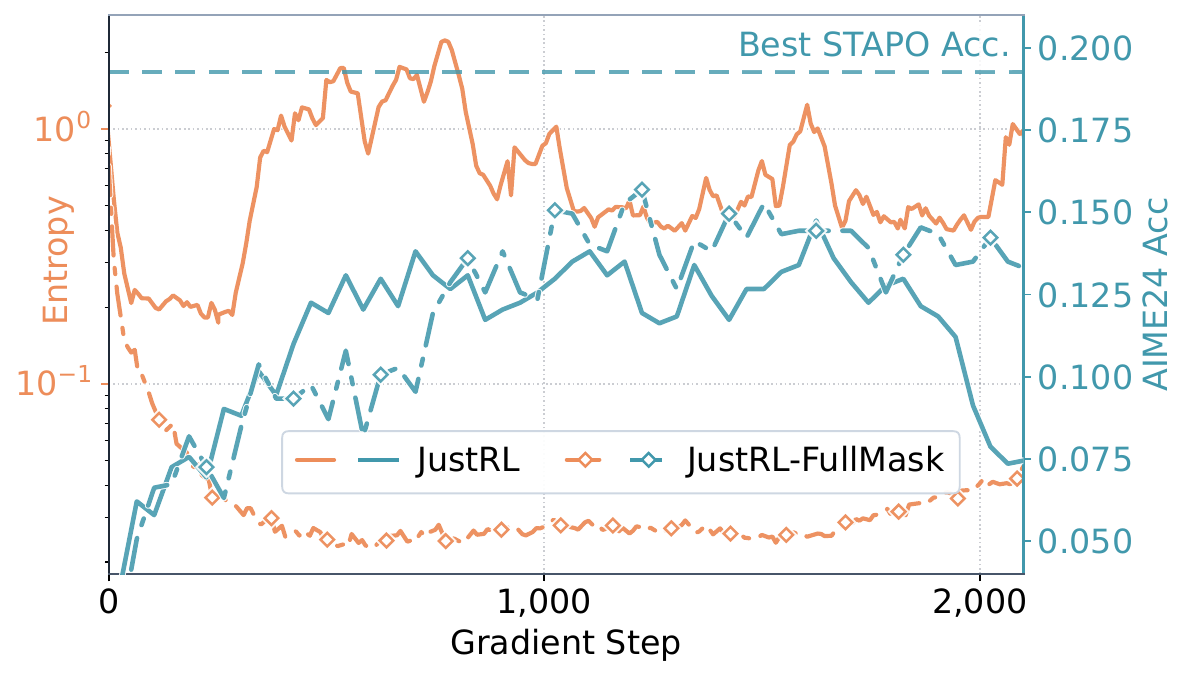}
  \vspace{-3mm}
  \caption{\textbf{Training curves of entropy and AIME24 accuracy.} JustRL-FullMask denotes a variant of JustRL that masks all low-probability tokens within correct responses.}
  \label{fig:entropy_aime24}
  \vspace{-1mm} % 微调图片底部的多余留白
\end{wrapfigure}
\begin{lemma}[Entropy Update Mechanism {\cite{xi2025bapo}}]
  \label{lemma:entropy_pg_general}
  Consider the language policy $\pi_\theta(\cdot)$ updated via a natural policy gradient step with learning rate $\eta$. Let $y'$ denote a generic token variable distributed according to the current policy $\pi_\theta(\cdot \mid \bm{x}, \bm{y}_{i,<t})$ at step $t$. The change in entropy $\Delta\mathcal{H}(y_{i,t})$ associated with a specific, already-sampled token $y_{i,t}$ between two consecutive policy iterates satisfies:
  \begin{equation}
    \begin{aligned}
      \Delta\mathcal{H}(y_{i,t}) \approx & -\eta \cdot \pi_\theta(y_{i,t} \mid \bm{x}, \bm{y}_{i,<t}) \cdot \Bigl[ \log \pi_\theta(y_{i,t} \mid \bm{x}, \bm{y}_{i,<t})                                                                                 \\
                                         & - \mathbb{E}_{y' \sim \pi_\theta} \bigl[ \log \pi_\theta(y' \mid \bm{x}, \bm{y}_{i,<t}) \bigr] \Bigr] \cdot \Bigl[ \hat{A}_i(y_{i,t}) - \mathbb{E}_{y' \sim \pi_\theta} \bigl[ \hat{A}_i(y') \bigr] \Bigr].
    \end{aligned}
  \end{equation}
\end{lemma}

As formalized in Lemma~\ref{lemma:entropy_pg_general}, low-probability tokens within correct responses induce positive entropy updates, thereby sustaining policy exploration.
This trade-off implies that the key challenge is to selectively remove truly detrimental spurious tokens among low-probability tokens while preserving stable exploration.

To operationalize this distinction, we partition the representation space into four quadrants based on token statistics (Figure~\ref{fig:quadrants}). Spurious tokens predominantly occupy the low-probability regime, demonstrating a clear concentration within low-entropy states. We attribute this to the underlying generation dynamics: low-probability tokens in high-entropy states generally signify legitimate exploration and are therefore structurally reasonable. Conversely, low-entropy (high-confidence) states inherently possess highly probable, valid candidates; the realization of a low-probability token in such contexts is merely an artifact of random sampling, rendering it anomalous and significantly elevating the risk of spurious generation.

Furthermore, we analyze the gradient norm at the token level during training.  Specifically, at decoding step $t$ of sequence $\bm{y}_i$, the LLM produces a logit vector $\bm{a}_{i,t} \in \mathbb{R}^{|\mathcal{V}|}$ over the vocabulary $\mathcal{V}$, inducing the policy distribution $\pi_\theta(\cdot \mid \bm{x}, \bm{y}_{i,<t})$ via the softmax function. We analyze the per-token gradient associated with $y_{i,t}$ as it propagates to intermediate layers.

\begin{theorem}[Policy Gradient Norm Bounds]\label{prop:grad_bounds}
  Consider the optimization objective at step $t$ for sample $i$ with target token $y_{i,t}$.
  The squared $\ell_2$-norm of the gradient $\nabla_{\bm{a}_{i,t}} \mathcal{J}$ w.r.t. the logits $\bm{a}_{i,t} \in \mathbb{R}^{|\mathcal{V}|}$ is bounded by the entropy $\mathcal{H}(\pi_{\theta})$ and the target probability $\pi_\theta(y_{i,t} \mid \bm{x}, \bm{y}_{i,<t})$ as follows:
  \begin{equation}
    \begin{aligned}
      |w_{i,t}|^2 \Bigl( 1 - 2\pi_\theta(y_{i,t} \mid \bm{x}, \bm{y}_{i,<t}) + e^{-\mathcal{H}(\pi_{\theta})} \Bigr)
       & \leq \|\nabla_{\bm{a}_{i,t}} \mathcal{J}\|^2                                                                          \\
       & \leq |w_{i,t}|^2 \Bigl( 2 - 2\pi_\theta(y_{i,t} \mid \bm{x}, \bm{y}_{i,<t}) - C_V \mathcal{H}(\pi_{\theta})^2 \Bigr),
    \end{aligned}
  \end{equation}
  where $C_V = \frac{|\mathcal{V}| - 1}{|\mathcal{V}| (\ln |\mathcal{V}|)^2}$ and the scaling weight $w_{i,t}$ is defined as:
  \begin{equation}
    w_{i,t} =
    \begin{cases}
      0,                                                                           & \mathrm{if } (\hat{A}_i > 0 \land \rho_{i,t} > 1 + \epsilon_\mathrm{high}) \lor (\hat{A}_i < 0 \land \rho_{i,t} < 1 - \epsilon_\mathrm{low}), \\
      \frac{\pi_\theta(y_{i,t} \mid \bm{x}, \bm{y}_{i,<t})}
      {\pi_{\theta_{\mathrm{old}}}(y_{i,t} \mid \bm{x}, \bm{y}_{i,<t})} \hat{A}_i, & \mathrm{otherwise}.
    \end{cases}
  \end{equation}
\end{theorem}
\begin{proof}
  See Appendix~\ref{app:grad_bounds_proof} for the detailed derivation.
\end{proof}

Theorem~\ref{prop:grad_bounds} demonstrates that the tokens characterized by both low probability and low entropy induce a larger gradient norm, a phenomenon empirically corroborated in Figure~\ref{fig:mean_grad}. Such tokens are likely spurious, thereby introducing erroneous updates that can destabilize training.

To systematically analyze how different token types influence optimization, we focus on tokens within correct answers, which inherently possess a positive advantage ($\hat{A} > 0$), and categorize them along two binary axes: token probability (high/low) and entropy (high/low). As summarized in Table~\ref{tab:token_analysis}, tokens characterized by \emph{low probability and low entropy} consistently exhibit detrimental optimization effects across all evaluation criteria. First, such tokens are highly likely to be spurious, leading to incorrect update directions. Second, they are associated with disproportionately large gradient norms, thereby amplifying their adverse impact during optimization. Finally, they contribute to entropy explosion in standard training, further exacerbating training instability.
\begin{table}[t]
  \centering
  \renewcommand{\arraystretch}{1.2} % 行高增加 1/5
  \caption{\textbf{Taxonomy of Token-Level Optimization Mechanisms.}
    Left: Token properties.
    Right: Influence on optimization.
    \protect\colorbox{green!20}{Green} indicates beneficial updates,
    \protect\colorbox{yellow!20}{yellow} indicates intermediate updates,
    and \protect\colorbox{red!15}{red} indicates detrimental updates.
    $\tau_p$ and $\tau_h$ represent the thresholds for token probability and entropy, respectively.}
  \label{tab:token_analysis}
  \resizebox{0.9\textwidth}{!}{
    \begin{tabular}{ccc|ccc}
      \toprule
      \multicolumn{3}{c|}{\textbf{Token Properties}}
                & \multicolumn{3}{c}{\textbf{Influence on Optimization}}                                                               \\
      \cmidrule(lr){1-3} \cmidrule(lr){4-6}
      Advantage & Token Prob.                                            & Entropy
                & Spurious Risk                                          & Gradient Norm            & Entropy Change~\cite{xi2025bapo} \\
      \midrule

      % % 模块 1: Advantage < 0
      % \multirow{4}{*}{$\hat A < 0$}
      % & $\pi_t \ge \tau_p$ & $\mathcal{H} \ge \tau_h$
      % & \cellcolor{red!15}High 
      % & \cellcolor{green!20}Low 
      % & \cellcolor{red!15}$\uparrow$ \\

      % & $\pi_t \ge \tau_p$ & $\mathcal{H} < \tau_h$
      % & \cellcolor{yellow!20}Medium 
      % & \cellcolor{green!20}Low 
      % & \cellcolor{red!15}$\uparrow$ \\

      % & $\pi_t < \tau_p$ & $\mathcal{H} \ge \tau_h$
      % & \cellcolor{green!20}Low 
      % & \cellcolor{yellow!20}Medium 
      % & \cellcolor{green!20}$\downarrow$ \\

      % & $\pi_t < \tau_p$ & $\mathcal{H} < \tau_h$
      % & \cellcolor{green!20}Low 
      % & \cellcolor{red!15}High 
      % & \cellcolor{green!20}$\downarrow$ \\

      % \midrule

      % 模块 2: Advantage > 0
      \multirow{4}{*}{$\hat A > 0$}
                & $\pi_t \ge \tau_p$                                     & $\mathcal{H} \ge \tau_h$
                & \cellcolor{green!20}Low
                & \cellcolor{green!20}Low
                & \cellcolor{green!20}$\downarrow$                                                                                     \\

                & $\pi_t \ge \tau_p$                                     & $\mathcal{H} < \tau_h$
                & \cellcolor{green!20}Low
                & \cellcolor{yellow!20}Medium
                & \cellcolor{green!20}$\downarrow$                                                                                     \\

                & $\pi_t < \tau_p$                                       & $\mathcal{H} \ge \tau_h$
                & \cellcolor{yellow!20}Medium
                & \cellcolor{yellow!20}Medium
                & \cellcolor{red!15}$\uparrow$                                                                                         \\

      % 最后一行，满足交错规律并保持原高亮格式
      \rowcolor{gray!10}
                & $\pi_t < \tau_p$                                       & $\mathcal{H} < \tau_h$
                & \cellcolor{red!15}\textbf{High}
                & \cellcolor{red!15}\textbf{High}
                & \cellcolor{red!15}$\boldsymbol{\uparrow}$                                                                            \\

      \bottomrule
    \end{tabular}}
  \vspace{-4mm}
\end{table}

\subsection{STAPO with the Silencing Spurious Tokens (S2T) Mechanism}

Motivated by Table~\ref{tab:token_analysis}, we propose the \textbf{Silencing Spurious Tokens (S2T)} mechanism, which selectively discards the gradient contributions of tokens that fall into the destructive regime:

\begin{equation}
  \label{eq:indicator_s2t}
  \mathbb{I}^{\mathrm{S2T}}_{i,t} =
  \begin{cases}
    0, & \mathrm{if}\ \hat{A}_i > 0 \land \pi(y_{i,t}) < \tau_p \land \mathcal{H}_t < \tau_h^{(q)}, \\
    1, & \mathrm{otherwise},
  \end{cases}
\end{equation}

where $\tau_p$ is a fixed probability threshold that defines an absolute notion of rarity, and $\tau_h^{(q)}$ denotes the $q$-th quantile of entropy, computed dynamically over the low-probability tokens within correct responses in each mini-batch.

% {r} 表示图片靠右 (Right)
% {0.45\textwidth} 是图片占据的宽度，您可以根据需要调整
\begin{wrapfigure}{r}{0.40\textwidth}
  \vspace{-5mm} % 微调图片顶部的多余留白（视情况增减）
  \centering
  \includegraphics[width=\linewidth]{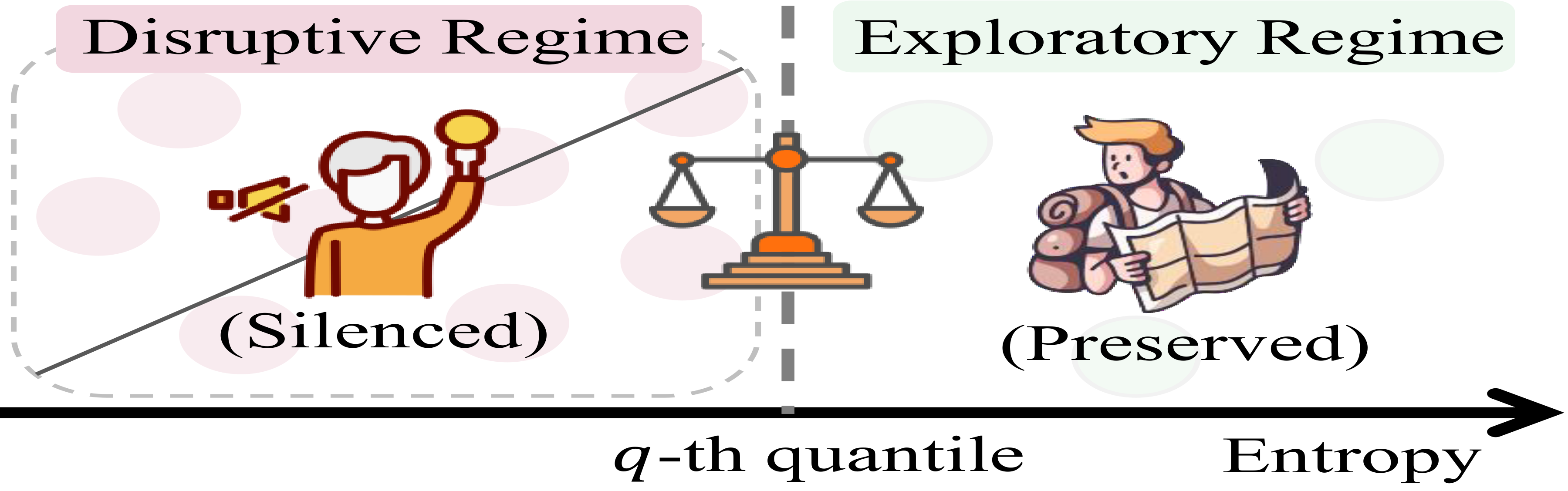}
  \vspace{-3mm}
  \caption{\textbf{Illustration of the Entropy Threshold.}}
  \label{fig:exploration}
  \vspace{-4mm} % 微调图片底部的多余留白
\end{wrapfigure}

The core idea is to first identify low-probability tokens within correct responses, which are characterized as exploratory tokens according to Lemma~\ref{lemma:entropy_pg_general}. Among these tokens, we further partition them based on entropy using the quantile threshold $\tau_h^{(q)}$. Tokens in the low-entropy region (i.e., the bottom $q$ fraction), which correspond to the most destructive cases identified in Table~\ref{tab:token_analysis}, are suppressed. In contrast, the remaining high-entropy tokens (the top $1-q$ fraction) are preserved to maintain sufficient exploration, as illustrated in Figure~\ref{fig:exploration}. This adaptive Silencing mechanism achieves a principled balance between stabilizing optimization and preserving exploration capacity. To provide further insight, we include word cloud comparisons for both types of tokens in Section~\ref{sec:word_cloud}.

In practice, we approximate the ideal mask $\mathbb{I}^{\mathrm{Spurious}}$ using the S2T mechanism, i.e., $\mathbb{I}^{\mathrm{Spurious}}_{i,t} \approx \mathbb{I}^{\mathrm{S2T}}_{i,t}$ in Equation~\eqref{eq:stapo_loss}, resulting in a computationally efficient instantiation of STAPO. The complete training procedure is outlined in Appendix~\ref{app:algorithm}.

\section{Related Work}
% \textbf{Reinforcement Learning for LLMs.} 
% In the era of small-model training, various effective RL algorithms have been proposed, including TD3~\cite{fujimoto2018addressing}, TRPO~\cite{schulman2015trust}, DSAC~\cite{10858686}, and BOOM~\cite{zhanbootstrap}.
% Recently,
% Reinforcement Learning from human feedback (RLHF) has emerged as a predominant approach for aligning LLMs with human preferences and diverse downstream objectives~\cite{jaech2024openai, guo2025deepseek}. While early approaches relied on on-policy optimization such as PPO~\cite{schulman2017proximal}, recent work has shifted toward more efficient preference-based Direct Preference Optimization (DPO)~\cite{rafailov2023direct}, which avoids explicit reward modeling and online rollouts. The focus has further expanded from general alignment to improving reasoning capabilities, motivating RL and training schemes such as GRPO~\cite{shao2024deepseekmath}. In this context, a range of policy optimization variants, including DAPO~\cite{yu2025dapo}, GSPO~\cite{zheng2025groupsequencepolicyoptimization}, SAPO~\cite{gao2025soft}, and other related methods~\cite{yue2025vapo, wang2025aspo}, have been developed to enhance optimization stability, sample efficiency, and scalability.

\textbf{Entropy Instability in RL.}
Entropy is a central issue in RL, as it directly affects exploration and training stability~\cite{duan2021distributional, wang2024diffusion}. In reasoning-oriented language models trained with RL~\cite{Hermann2026}, a persistent challenge is the rapid collapse of policy entropy during the early stages of optimization, which often leads to premature convergence and degraded performance. Prior work mitigates this issue through interventions such as selectively regularizing high-entropy tokens~\cite{wang2025beyond}, increasing the proportion of entropy-enhancing samples~\cite{xi2025bapo}, and modifying clipping strategies in policy optimization~\cite{yu2025dapo, chen2025minimax, gao2025soft}. However, these methods often introduce the opposite failure mode, where entropy grows excessively or becomes unstable, degrading reasoning coherence and leading to repetitive or unstructured outputs. Although some studies analyze entropy dynamics during RL training~\cite{cui2025entropy, wang2026entropy} and others explicitly enforce entropy stability as an optimization objective~\cite{yang2025entropic}, existing approaches largely treat entropy as a surface-level training signal rather than addressing the underlying sources of instability.

\textbf{Gradient Domination by Low-Probability Tokens.}
While various policy gradients aim to improve RL stability~\cite{zhan2025bicriteria, zhan2023continuous}, instability often stems from the disproportionate gradient influence of low-probability tokens. Rare tokens can generate excessively large updates~\cite{yang2025not}, allowing unstable predictions to dominate fine-tuning optimization. Recent studies~\cite{yang2025not,Cai2026} address this issue through probability-aware modulation. Token-Regulated GRPO (TR-GRPO)~\cite{le2025token} employs probability shaping to suppress sharp update directions arising from low-confidence predictions. To avoid discarding informative exploratory signals, Low-Probability Regularization (Lp-Reg)~\cite{huang2025low} filters out noise while preserving meaningful low-probability tokens. However, relying mainly on scalar probability thresholds, these approaches lack a joint, fine-grained treatment of confidence and probability, and thus fail to distinguish useful exploration from aleatoric noise.

\section{Experiments}
\label{sec:experiments}

\subsection{Settings}
\label{subsec:settings}
\noindent \textbf{Baselines.} We compare our approach against several RL algorithms for LLMs, including GRPO~\cite{shao2024deepseekmath}, 20-Entropy~\cite{wang2025beyond}, and JustRL~\cite{he2025justrl}. For all baselines, we follow the parameter settings reported in their original papers. Unless otherwise specified, STAPO uses $\tau_p = 0.002$ and $q = 75\%$ across all experiments. For a fair comparison, we do not apply the dynamic sampling technique introduced in~\cite{yu2025dapo} to any method. Detailed training settings are provided in Appendix~\ref{app:traing_details}.

\noindent \textbf{Benchmarks.}
For mathematical reasoning, we evaluate three Qwen3
base models (1.7B, 8B, and 14B) on six benchmarks:
AIME24~\cite{li2024numinamath},
AIME25~\cite{opencompass2025aime},
AMC23~\cite{li2024numinamath},
MATH500~\cite{hendrycks2021math},
Minerva~\cite{lewkowycz2022solving}, and
OlympiadBench~\cite{he2024olympiadbench}.
For code reasoning, we train Qwen3-1.7B-Base on
Eurus-RL-Code~\cite{cui2025process} and evaluate on HumanEval+, MBPP+~\cite{MBPP},
and LiveCodeBench v6~\cite{LiveCodeBench}.

\noindent \textbf{Evaluation Protocol.}
For mathematical benchmarks, we sample $N$ independent
responses per problem ($N=4$ for MATH500, Minerva,
and OlympiadBench; $N=32$ for the others) under two
decoding configurations: temperature $\rho_{\mathrm{T}}=1.0$
with top-p=1.0, and $\rho_{\mathrm{T}}=0.7$ with
top-p=0.9. The maximum generation length is 20,480
tokens, and the latter configuration is used by
default unless otherwise specified.
We use CompassVerifier-3B~\cite{liu2025compassverifier}
to correct errors in rule-based verification.
For code benchmarks, following the default decoding configuration
(temperature $\rho_{\mathrm{T}}=0.7$, top-p=0.9),
we report avg@4 over four independently generated
responses per problem.

\subsection{Main Results}\label{subsec:main_results}

\subsubsection{Training Behavior Analysis}
\textbf{Entropy Analysis.}
Entropy curves are a key indicator of learning progress in LLM training. As shown in Figure~\hyperref[fig:main_results_1_7b]{\ref*{fig:main_results_1_7b}(b)}, JustRL and 20-Entropy suffer from entropy explosion, while GRPO exhibits entropy collapse. Conversely, STAPO maintains a stable, well-regulated entropy profile after warm-up.
Token-level quantile analysis (Figure~\hyperref[fig:main_results_1_7b]{\ref*{fig:main_results_1_7b}(d)}) explains this stability: STAPO concentrates entropy exclusively at high quantiles (e.g., above the 80th percentile), keeping most tokens deterministic. This strategically preserves exploration for critical tokens while enhancing reasoning precision. In contrast, the uniformly high entropy in JustRL and 20-Entropy triggers repetitive generation, whereas GRPO's persistently low entropy stifles exploration and limits learning capacity.

\textbf{Performance Comparison.} Beyond maintaining a stable entropy profile and preserving critical exploratory tokens, STAPO demonstrates superior overall performance. During RL training, as shown in Figure~\hyperref[fig:main_results_1_7b]{\ref*{fig:main_results_1_7b}(c)}, STAPO achieves the highest reward, maintaining a clear advantage over all baselines. This suggests that STAPO successfully explores and acquires reasoning capabilities that other methods fail to capture. Furthermore, Figure~\hyperref[fig:main_results_1_7b]{\ref*{fig:main_results_1_7b}(a)} illustrates that STAPO not only attains the highest accuracy on AIME24 but also sustains strong performance without noticeable degradation even after 3000 steps, evidencing remarkable training stability. 
Ultimately, the stark contrast with JustRL underscores the detrimental impact of spurious tokens in baseline methods, corroborating the validity of our prior analysis in Section~\ref{subsec:eg_coupling_analysis}.

\begin{table}[!t]
\centering 
% Caption 保持不变
\caption{\textbf{Results on Math Reasoning Benchmarks across Three Models.} Each cell reports performance under the training-aligned configuration (temperature $\rho_{\mathrm{T}}$=1.0, top-p=1.0) and the JustRL~\cite{he2025justrl} evaluation setting (temperature $\rho_{\mathrm{T}}$=0.7, top-p=0.9, shown in gray). 
For each configuration, the best result is highlighted in \textbf{bold} (training-aligned setting) and in {\color{darkgray}dark gray} (JustRL setting), respectively.
}

\footnotesize
\label{tab:main_results}

% 处理 booktabs 和 rowcolor 潜在冲突的设置
\setlength{\aboverulesep}{0pt}
\setlength{\belowrulesep}{0pt}

% === 核心对齐设置 ===
\setlength{\tabcolsep}{1pt} % 让 extsep 来完全接管列间距的计算
\renewcommand{\arraystretch}{1.6} % 将原本过大的 1.6 行高降下来，使表格更紧凑

% 使用 tabular* 并配合 \extracolsep{\fill} 实现精确拉伸对齐
\begin{tabular*}{\textwidth}{@{\extracolsep{\fill}} l cccccc c @{}} 
\toprule
\textbf{Baseline} & \textbf{AIME24} & \textbf{AIME25} & \textbf{AMC23} & \textbf{MATH500} & \textbf{Minerva} & \textbf{Olympiad} & \cellcolor{gray!8}\textbf{Avg} \\ \hline

% --- Qwen3-1.7B ---
\rowcolor{headerblue} \multicolumn{8}{c}{\textbf{RL from the Qwen3-1.7B Base Model}} \\ \hline
GRPO       & \tabscore{7.60}{10.42}   & \tabscore{3.33}{3.33}   & \tabscore{37.81}{39.84} & \tabscore{65.90}{66.40} & \tabscore{\textbf{33.73}}{\color{darkgray}34.83} & \tabscore{31.16}{31.42} & \cellcolor{gray!8}\tabscore{29.92}{31.04} \\
20-Entropy & \tabscore{11.98}{17.40} & \tabscore{6.67}{11.67}  & \tabscore{42.66}{47.58} & \tabscore{63.70}{64.90} & \tabscore{21.23}{24.45} & \tabscore{31.34}{32.31} & \cellcolor{gray!8}\tabscore{29.60}{33.05} \\
JustRL     & \tabscore{9.17}{15.73}  & \tabscore{6.67}{9.28}  & \tabscore{43.20}{48.28} & \tabscore{66.25}{70.10} & \tabscore{30.61}{32.08} & \tabscore{30.71}{32.90} & \cellcolor{gray!8}\tabscore{31.10}{34.73} \\ \hline
\textbf{STAPO} & \tabscore{\textbf{18.44}}{\color{darkgray}19.27} & \tabscore{\textbf{10.00}}{\color{darkgray}13.33} & \tabscore{\textbf{54.22}}{\color{darkgray}53.12} & \tabscore{\textbf{73.10}}{\color{darkgray}71.80} & \tabscore{28.40}{26.10} & \tabscore{\textbf{36.94}}{\color{darkgray}36.80} & \cellcolor{gray!8}\tabscore{\textbf{36.85}}{\color{darkgray}36.74} \\ \hline

% --- Qwen3-8B ---
\rowcolor{headerblue} \multicolumn{8}{c}{\textbf{RL from the Qwen3-8B Base Model}} \\ \hline
GRPO       & \tabscore{31.25}{32.50} & \tabscore{24.69}{24.17} & \tabscore{75.23}{74.14} & \tabscore{88.90}{88.85} & \tabscore{55.88}{53.58} & \tabscore{61.50}{58.34} & \cellcolor{gray!8}\tabscore{56.24}{55.26} \\
20-Entropy & \tabscore{31.25}{36.15} & \tabscore{27.50}{27.39} & \tabscore{79.92}{80.00} & \tabscore{89.85}{89.85} & \tabscore{54.78}{54.23} & \tabscore{62.50}{60.39} & \cellcolor{gray!8}\tabscore{57.63}{58.00} \\
JustRL     & \tabscore{25.21}{35.83} & \tabscore{21.98}{26.04} & \tabscore{73.52}{\color{darkgray}81.48} & \tabscore{84.90}{87.55} & \tabscore{47.33}{51.03} & \tabscore{51.26}{56.57} & \cellcolor{gray!8}\tabscore{50.70}{56.42} \\ \hline
\textbf{STAPO} & \tabscore{\textbf{38.33}}{\color{darkgray}41.04} & \tabscore{\textbf{33.33}}{\color{darkgray}29.37} & \tabscore{\textbf{83.12}}{80.23} & \tabscore{\textbf{90.95}}{\color{darkgray}91.40} & \tabscore{\textbf{56.80}}{\color{darkgray}55.70} & \tabscore{\textbf{63.20}}{\color{darkgray}63.32} & \cellcolor{gray!8}\tabscore{\textbf{60.96}}{\color{darkgray}60.18} \\ \hline

% --- Qwen3-14B ---
\rowcolor{headerblue} \multicolumn{8}{c}{\textbf{RL from the Qwen3-14B Base Model}} \\ \hline
GRPO       & \tabscore{41.04}{42.40} & \tabscore{35.00}{26.67} & \tabscore{80.08}{82.81} & \tabscore{91.60}{91.00} & \tabscore{57.35}{56.89} & \tabscore{63.87}{63.39} & \cellcolor{gray!8}\tabscore{61.49}{60.53} \\
20-Entropy & \tabscore{42.08}{48.96} & \tabscore{34.06}{39.48} & \tabscore{84.45}{87.73} & \tabscore{91.65}{92.30} & \tabscore{51.10}{54.78} & \tabscore{61.30}{63.95} & \cellcolor{gray!8}\tabscore{60.77}{64.53} \\
JustRL     & \tabscore{34.48}{50.83} & \tabscore{18.33}{38.33} & \tabscore{76.41}{89.06} & \tabscore{89.10}{93.30} & \tabscore{54.60}{\color{darkgray}59.93} & \tabscore{61.54}{68.81} & \cellcolor{gray!8}\tabscore{55.74}{66.71} \\ \hline
\textbf{STAPO} & \tabscore{\textbf{51.04}}{\color{darkgray}52.81} & \tabscore{\textbf{41.67}}{\color{darkgray}41.15} & \tabscore{\textbf{89.63}}{\color{darkgray}90.78} & \tabscore{\textbf{93.60}}{\color{darkgray}93.70} & \tabscore{\textbf{59.74}}{59.38} & \tabscore{\textbf{70.92}}{\color{darkgray}68.95} & \cellcolor{gray!8}\tabscore{\textbf{67.76}}{\color{darkgray}67.80} \\ \bottomrule
\end{tabular*}
\end{table}

\begin{figure}[t!]
% \vspace{-6pt}
    \centering
    \includegraphics[width=1.0\linewidth]{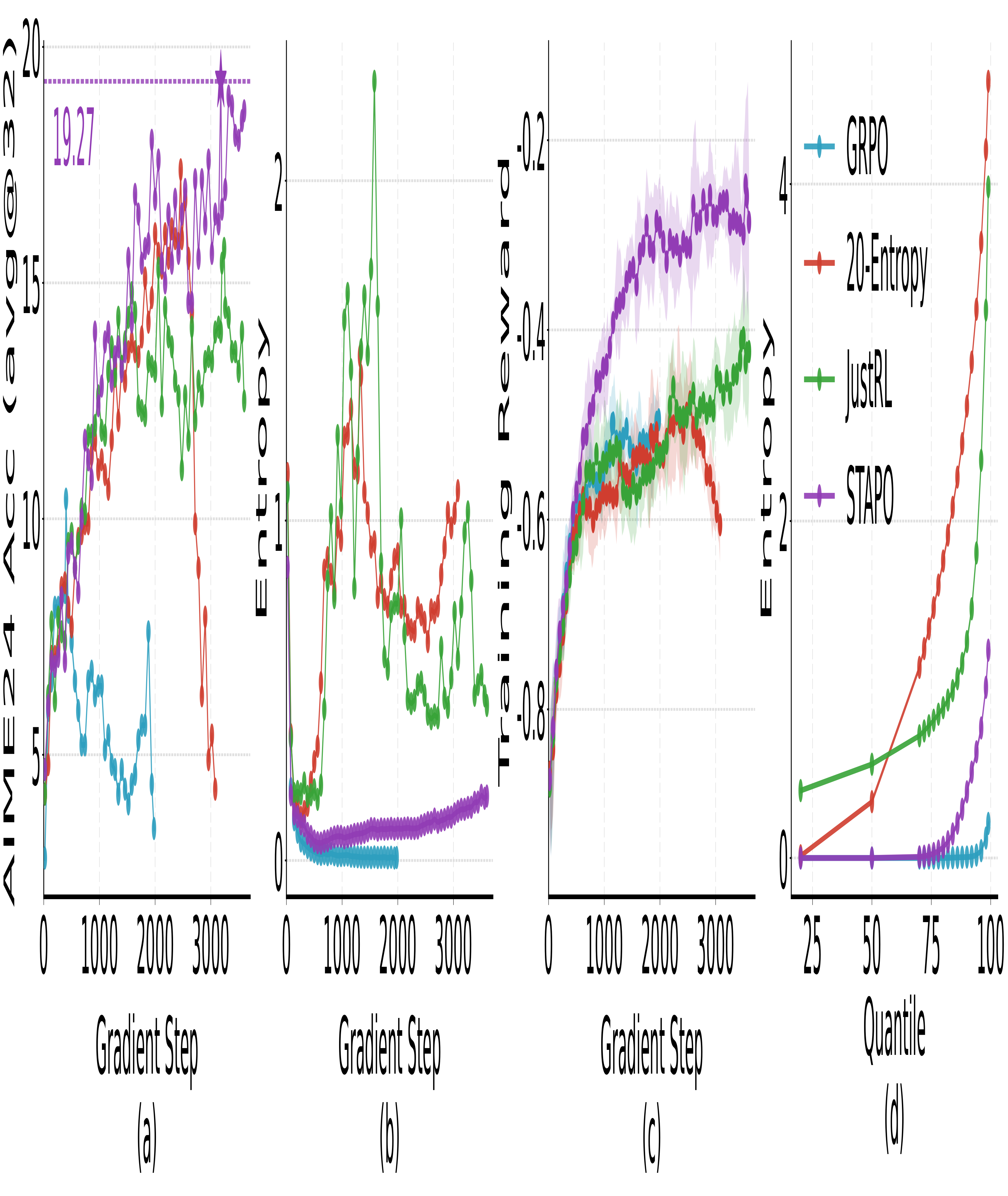} % 1.7B 规模
    \caption{\textbf{Training Curves on the Qwen3-1.7B Base Model.} Notably, STAPO achieves superior performance while maintaining stable policy entropy. Training dynamics for larger models (8B and 14B) exhibit consistent trends and are provided in Appendix~\ref{app:additional_training_dynamics}.}
    \label{fig:main_results_1_7b}
    % \vspace{-15pt}
\end{figure}

\subsubsection{Mathematical Reasoning Results}
Table~\ref{tab:main_results} summarizes performance across different model scales. In the training-aligned setting ($\rho_{\mathrm{T}}$=1.0, top-p=1.0), STAPO demonstrates excellent scalability and enhanced intrinsic reasoning capability, surpassing the strongest baselines with average relative accuracy improvements of 18.49\%, 5.78\%, and 10.20\% at the 1.7B, 8B, and 14B scales, respectively.

% {r} 表示图片靠右 (Right)
% {0.45\textwidth} 是图片占据的宽度，您可以根据需要调整
\begin{wrapfigure}{r}{0.42\textwidth}
  \vspace{-0mm} % 微调图片顶部的多余留白（视情况增减）
  \centering
  \includegraphics[width=\linewidth]{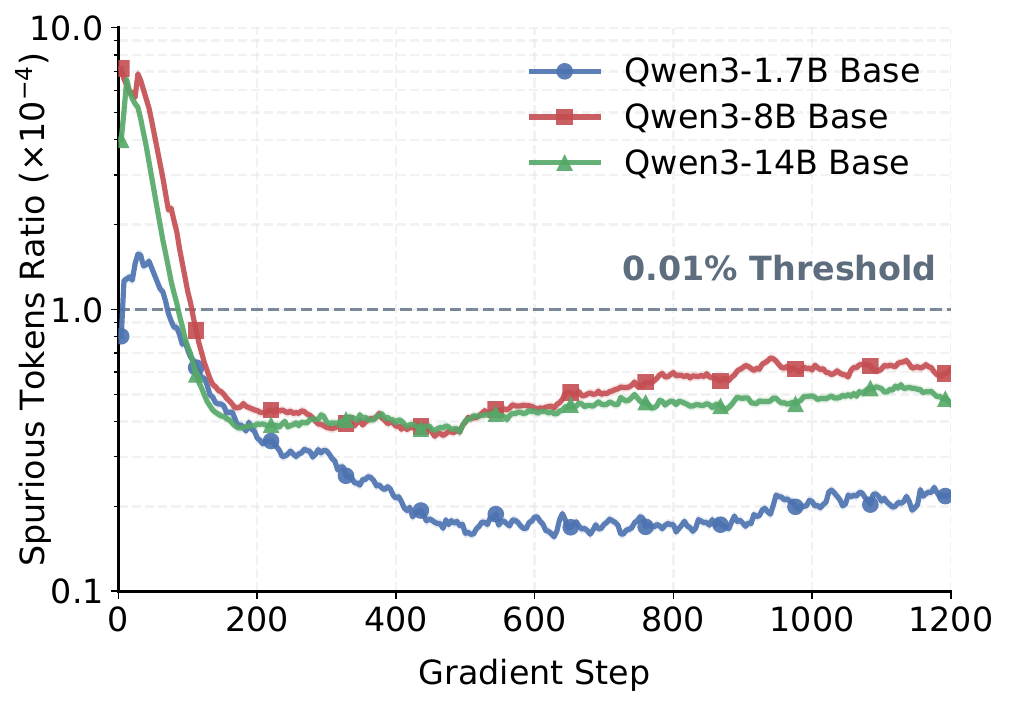}
  \vspace{-3mm}
\caption{\textbf{Spurious Token Ratio Curves.} After warm-up, the ratio of spurious tokens stays below 0.01\% across all models.}
  \label{fig:spurious_ratio_independent}
  \vspace{-4mm} % 微调图片底部的多余留白
\end{wrapfigure}

When evaluated under the JustRL configuration ($\rho_{\mathrm{T}}$=0.7, top-p=0.9), STAPO continues to achieve state-of-the-art performance. The performance gap narrows slightly because high-entropy baselines (JustRL and 20-Entropy) benefit more from decoding heuristics that suppress tail generations. Crucially, STAPO consistently attains optimal results across varied evaluation settings, demonstrating strong robustness and overall performance. This confirms that its inherently stable distribution is less dependent on such decoding heuristics.

Notably, as shown in Figure~\ref{fig:spurious_ratio_independent}, these significant gains are achieved by masking a negligible fraction of tokens (mostly $<0.01\%$) across the 1.7B, 8B, and 14B models. This indicates that RL instability is caused by sparse spurious tokens, which disproportionately affect gradient updates and are precisely identified and effectively suppressed by STAPO.
To better illustrate the impact of spurious tokens, we also provide additional examples, along with a simple case classification, in Appendix~\ref{app:spurious_examples}.

\subsubsection{Code Reasoning Results}
\label{subsec:code_results}

\begin{wraptable}{r}{0.46\columnwidth}
    \vspace{-5mm}
    \centering
\caption{\textbf{Code Reasoning Results with Qwen3-1.7B-Base.}}
    \label{tab:code_results}
    % \vspace{1mm}
    \scriptsize
    \setlength{\tabcolsep}{2pt}
    \renewcommand{\arraystretch}{1.15}
    \resizebox{\linewidth}{!}{%
    \begin{tabular}{@{}lcccc@{}}
        \toprule
        \textbf{Method}
        & \textbf{HumanEval+}
        & \textbf{MBPP+}
        & \textbf{LCB v6}
        & \cellcolor{gray!8}\textbf{Avg.} \\
        \midrule
        GRPO
        & 58.23 & 54.76 & 8.29
        & \cellcolor{gray!8}40.43 \\
        20-Entropy
        & 59.15 & 55.95 & 12.43
        & \cellcolor{gray!8}42.51 \\
        JustRL
        & 55.79 & 55.69 & 13.86
        & \cellcolor{gray!8}41.78 \\
        \midrule
        \textbf{STAPO}
        & \textbf{63.41}
        & \textbf{56.35}
        & \textbf{15.14}
        & \cellcolor{gray!8}\textbf{44.97} \\
        \bottomrule
    \end{tabular}%
    }
    % \vspace{-4mm}
\end{wraptable}
As shown in Table~\ref{tab:code_results},
STAPO consistently achieves the best performance
across all three coding benchmarks, outperforming
the strongest baseline by 5.79\% in relative average
accuracy.
Notably, these improvements are achieved without
any task-specific hyperparameter tuning, demonstrating
the generalizability of STAPO beyond mathematical
reasoning.
Furthermore, our qualitative analysis reveals that
even positively rewarded coding responses can contain
locally spurious tokens, such as flawed intermediate
code, invalid syntax, and redundant formatting.
We provide representative examples with detailed
analyses in Appendix~\ref{app:code_spurious_cases},
further illustrating that the spurious-token
phenomenon is not limited to mathematical reasoning.

\subsection{Ablation and Sensitivity Analysis}
\label{sec:ablation_sensitivity}

\textbf{Masking Strategy Ablation.} We first ablate different masking strategies applied to positive-advantage, low-probability tokens on the Qwen3-1.7B-Base model. As shown in Figure~\ref{fig:ablation_independent}, low-entropy masking (STAPO) consistently achieves the best performance, whereas high-entropy masking significantly degrades results, suggesting that masking high-entropy tokens can suppress useful exploratory signals.
Random masking yields intermediate performance, gradually approaching that of low-entropy masking as the masking ratio increases. This indicates that while low-entropy masking remains the optimal strategy, random masking can function as a heuristic approximation when the masking budget is sufficiently large.

\textbf{Hyperparameter Sensitivity.} We further analyze the model's sensitivity to two key hyperparameters: the entropy quantile $q$ and the probability threshold $\tau_p$. As shown in Figure~\ref{fig:ablation_independent}, decreasing $q$ leads to a rapid degradation in performance for both high-entropy and random masking strategies. In contrast, STAPO exhibits only a marginal decline, demonstrating superior robustness to variations in the entropy quantile.
Figure~\ref{fig:tau_p} illustrates the effect of varying $\tau_p$, which establishes the threshold for identifying low-probability tokens. An insufficiently low $\tau_p$ may fail to isolate spurious tokens, leading to increased training entropy and subsequent performance degradation. Conversely, an excessively high $\tau_p$ may erroneously penalize legitimate tokens, rendering the policy overly conservative and impairing its reasoning capabilities. Despite these effects, overall performance remains relatively stable, indicating that the model is largely insensitive to $\tau_p$.

\subsection{Word Cloud Visualization}
\label{sec:word_cloud}
During the training of Qwen3-1.7B, we visualized word clouds for low-probability correct answers, categorizing tokens into low-entropy "spurious tokens" and high-entropy "exploratory tokens" as illustrated in Figures~\ref{fig:spurious_tokens_cloud} and~\ref{fig:non_spurious_tokens_cloud}. The most frequent spurious tokens primarily include specific digits (e.g., ``1'', ``2'', ``3'') and mathematical symbols (e.g., ``\$'', ``x''). While these tokens may appear in correct responses, they are often associated with formatting or calculation errors. When combined with large gradient updates, they can induce disproportionately unstable optimization dynamics.
In contrast, the retained exploratory tokens correspond to the structural vocabulary of mathematical reasoning, including words like ``Let'', ``This'', ``So'', and ``Wait''. These tokens represent key procedural and logical components that sustain coherent reasoning chains. Retaining these tokens ensures that the model's core logical reasoning capabilities remain intact.

\begin{figure*}[t!]
    \centering
    
    % ==========================================
    % 第一层：只放图片，强制底部对齐 [b]（左消融，右比例）
    % ==========================================
    \begin{minipage}[b]{0.48\textwidth}
        \centering
        % 现在左边是柱状图（消融实验）
        \includegraphics[width=\linewidth]{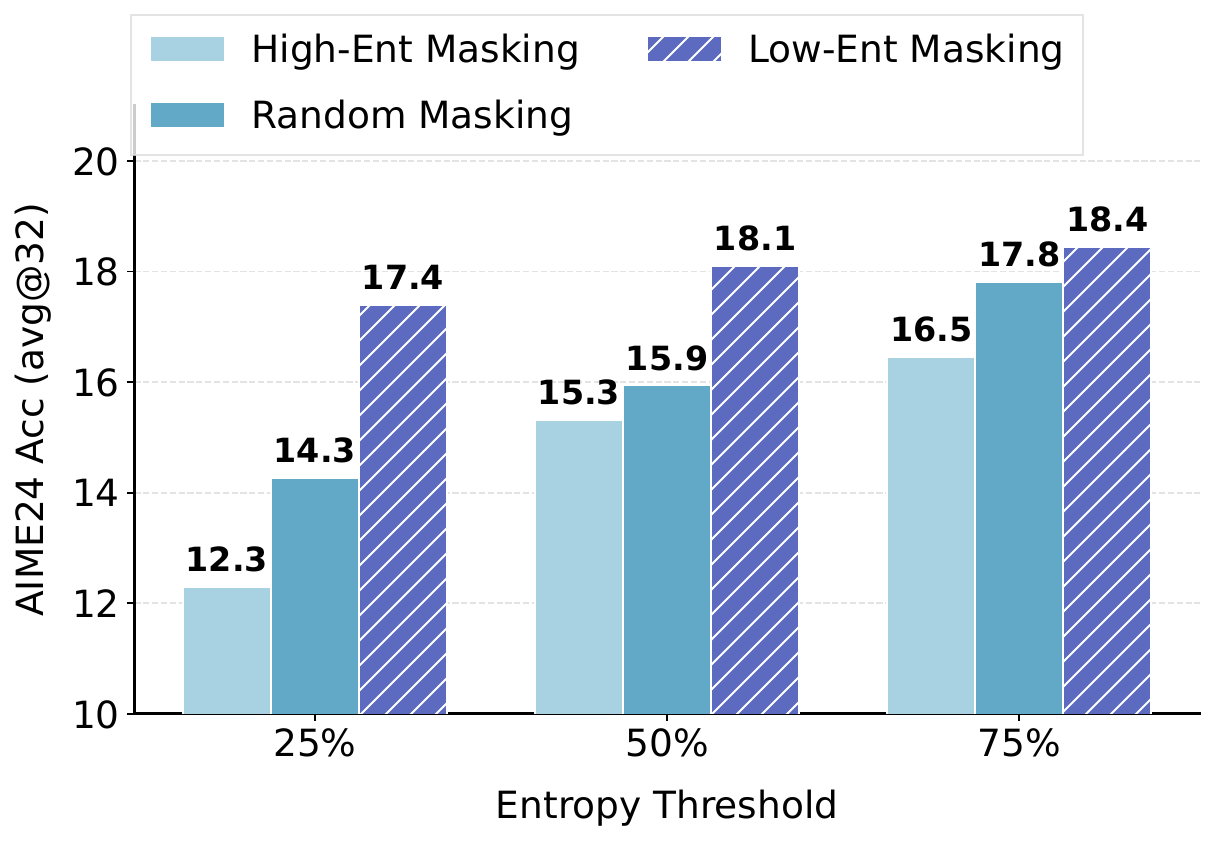}
    \end{minipage}%
    \hfill
    \begin{minipage}[b]{0.48\textwidth}
        \centering
        % 现在右边是折线图（Spurious Ratio）
        \includegraphics[width=\linewidth]{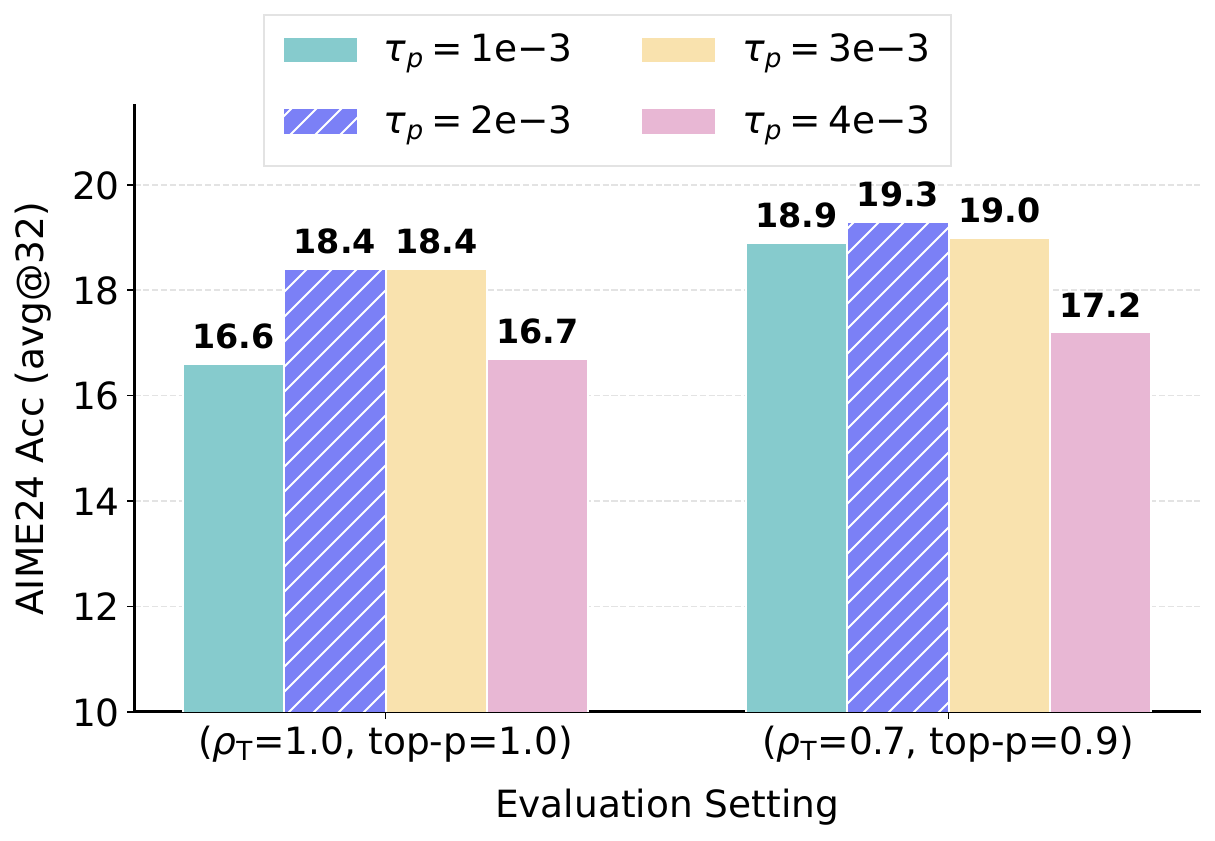}
    \end{minipage}
    
    % ==========================================
    % --- 压缩图片和标题之间的间距 ---
    \vspace{-2mm} 
    % ==========================================
    
    % 第二层：只放标题，强制顶部对齐 [t]（左消融，右比例）
    \begin{minipage}[t]{0.48\textwidth}
        \caption{\textbf{Masking Strategy Ablation.} Various entropy masking strategies are applied to positive-advantage, low-probability tokens on the 1.7B base model with $\rho_{\mathrm{T}}$=1.0,  top-p=1.0.}
        \label{fig:ablation_independent}
    \end{minipage}%
    \hfill
    \begin{minipage}[t]{0.48\textwidth}
        \caption{\textbf{Sensitivity of AIME24 Accuracy to $\tau_p$.} The AIME24 Acc (avg@32) is evaluated across different probability thresholds in various settings on the 1.7B base model.}
        \label{fig:tau_p}
    \end{minipage}

    % ==========================================
    % --- 压缩标题（图底）和正文之间的间距 ---
    \vspace{-4mm}
    % ==========================================
    
\end{figure*}

\begin{figure*}[t!]
    \centering
    % --- 子图 (a)：折线图 ---
    % --- 子图 (b)：Spurious Tokens 词云 ---
\begin{subfigure}{0.45\textwidth}
        \centering
        \includegraphics[width=\textwidth]{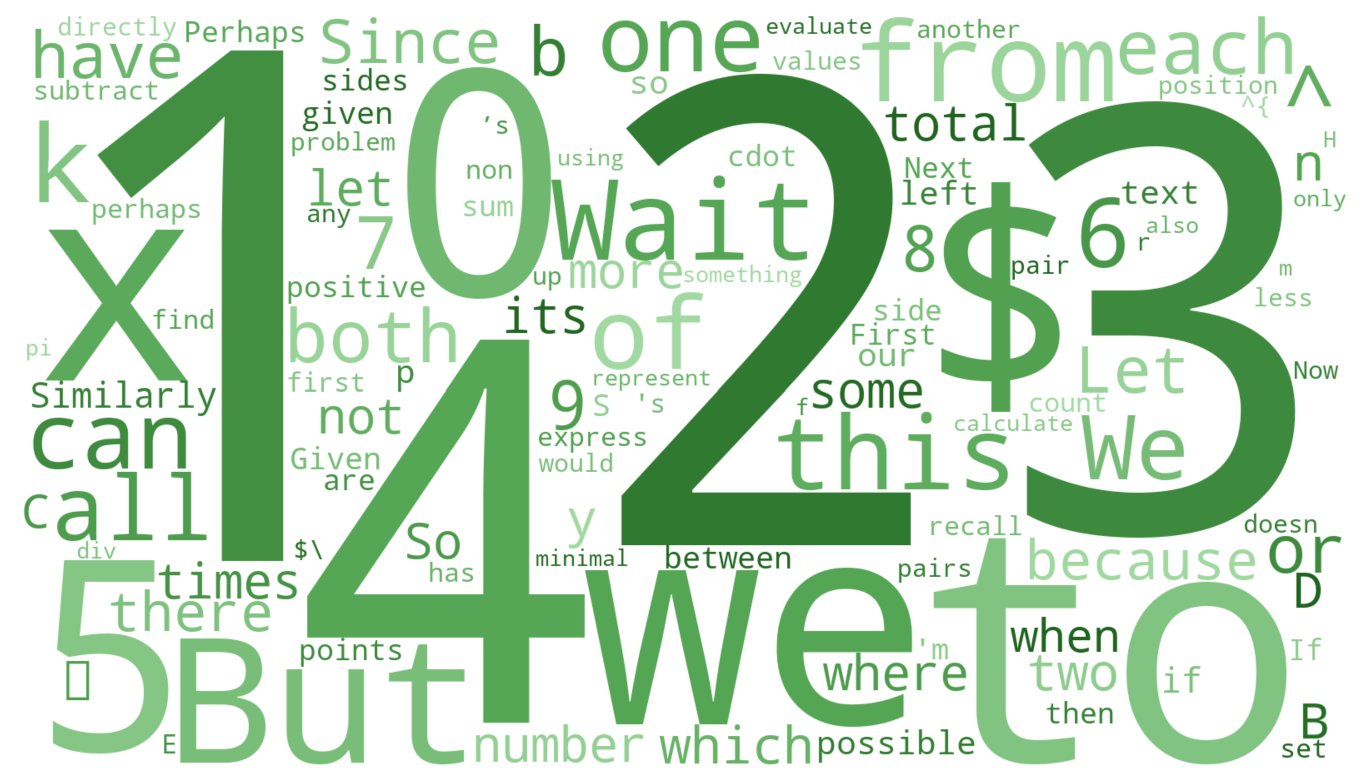}
        \caption{Spurious tokens.}
        \label{fig:spurious_tokens_cloud}
    \end{subfigure}
    \hspace{1cm}
    % --- 子图 (b)：Non-spurious Tokens 词云 (Low Prob, High Entropy) ---
    \begin{subfigure}{0.45\textwidth}
        \centering
        \includegraphics[width=\textwidth]{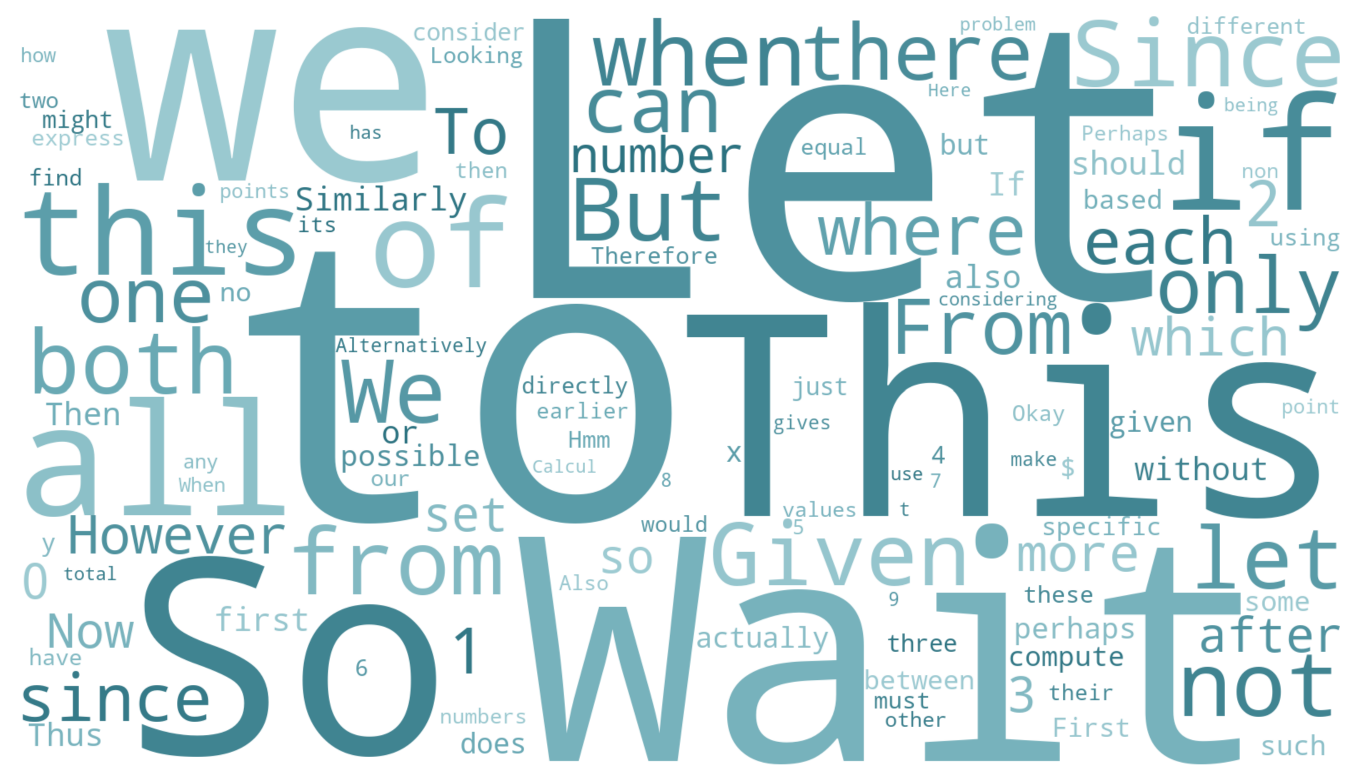}
        \caption{Exploratory tokens.}
        \label{fig:non_spurious_tokens_cloud}
    \end{subfigure}

\caption{\textbf{Word clouds.} The font size of each token is proportional to its occurrence frequency. To facilitate a clearer analysis, common articles and prepositions are filtered out.}
    \vspace{-4mm}
\label{fig:spurious_analysis_tokens}
\end{figure*}

\section{Conclusion}
\label{Section:conclusion}

In this work, we formally introduce the concept of spurious tokens in RL training, demonstrating that a small fraction of these uninformative tokens can disproportionately skew optimization dynamics. To address this issue, we propose STAPO as a general paradigm designed to mitigate its detrimental effects. Within this paradigm, we implement the S2T mechanism as a specific instantiation to detect and curtail the gradient influence of spurious tokens. Empirical results show that STAPO substantially stabilizes policy entropy and consistently improves reasoning performance.

In future work, we plan to broaden this scope by analyzing tokens from incorrect responses and extending our empirical evaluation to diverse task domains. We hope that these efforts will provide further insights into improving the stability of RL for LLM reasoning.

\section*{Funding and Competing Interests}
This study is supported by the National Natural Science Foundation of China (NSFC) under Grant No. 92582205 and the Shanghai Automotive Industry Science and Technology Development Foundation.
The authors declare no competing interests.

\clearpage

\bibliographystyle{unsrt}
\bibliography{main}

\clearpage

\beginappendix

\section{Gradient Norm Decomposition}
\label{app:grad_decomp}

We first establish an exact decomposition of the gradient norm with respect to the logits, which serves as the foundation for our bounds.
\begin{lemma}[Gradient Norm Decomposition~\cite{yang2025not}]
\label{lem:grad_norm_decomp}
Let $\bm{a}_{i,t} \in \mathbb R^{|\mathcal{V}|}$ denote the logits and $\pi_\theta(v^n \mid \bm{x}, \bm{y}_{i,<t}) = \frac{e^{a^n}}{\sum_m e^{a^m}}$ be the induced softmax distribution. Let $y_{i,t}$ be the target token with index $k$ (i.e., $v^k = y_{i,t}$).
The squared $\ell_2$-norm of the gradient of the group-style objective $\mathcal{J}$ with respect to $\bm{a}_{i,t}$ satisfies:
\begin{equation}
\label{eq:grad_norm_decomp}
\|\nabla_{\bm{a}_{i,t}} \mathcal{J}(y_{i,t})\|^2
=
|w_{i,t}|^2
\left(
1 - 2\pi_\theta(y_{i,t} \mid \bm{x}, \bm{y}_{i,<t}) + \sum_{n=1}^{|\mathcal{V}|}\pi_\theta(v^n)^2
\right),
\end{equation}
where the weight $w_{i,t}$ is defined as:
\begin{equation}
w_{i,t} = 
\begin{cases} 
0, & \mathrm{if } (\hat{A}_i > 0 \land \rho_{i,t} > 1 + \epsilon) \lor (\hat{A}_i < 0 \land \rho_{i,t} < 1 - \epsilon), \\
\frac{\pi_\theta(y_{i,t} \mid \bm{x}, \bm{y}_{i,<t})}
{\pi_{\theta_{\mathrm{old}}}(y_{i,t} \mid \bm{x}, \bm{y}_{i,<t})} \hat{A}_i, & \mathrm{otherwise}.
\end{cases}
\end{equation}
\end{lemma}

\begin{proof}
Recall that the gradient of the log-likelihood $\ln \pi_\theta(y_{i,t})$ with respect to the logit $a^n$ is given by $\delta_{kn} - \pi_\theta(v^n)$, where $\delta_{kn}$ is the Kronecker delta. The gradient of the weighted objective $\mathcal{J}$ is therefore:
\begin{equation}
\frac{\partial \mathcal{J}(y_{i,t})}{\partial a^n} 
= w_{i,t} \cdot (\delta_{kn} - \pi_\theta(v^n))
= 
\begin{cases} 
w_{i,t} (1 - \pi_\theta(v^k)), & \text{if } n = k, \\
-w_{i,t} \pi_\theta(v^n), & \text{if } n \neq k.
\end{cases}
\end{equation}
Computing the squared $\ell_2$-norm by summing over all $n \in \{1, \dots, |\mathcal{V}|\}$:
\begin{equation}
\begin{aligned}
\|\nabla_{\bm{a}_{i,t}} \mathcal{J}(y_{i,t})\|^2 
&= \left( w_{i,t} (1 - \pi_\theta(v^k)) \right)^2 + \sum_{n \neq k} \left( -w_{i,t} \pi_\theta(v^n) \right)^2 \\
&= |w_{i,t}|^2 \left( 1 - 2\pi_\theta(v^k) + {\pi_\theta(v^k)^2 + \sum_{n \neq k} \pi_\theta(v^n)^2}\right),\\
&= |w_{i,t}|^2 \left( 1 - 2\pi_\theta(v^k) + \sum_{n=1}^{|\mathcal{V}|} \pi_\theta(v^n)^2\right).
\end{aligned}
\end{equation}
Identifying $v^k$ as $y_{i,t}$ completes the proof.
\end{proof}

\section{Proof of Theorem~\ref{prop:grad_bounds}}
\label{app:grad_bounds_proof}
\begin{proof}
We now complete the proof of Theorem~\ref{prop:grad_bounds}.
By Lemma~\ref{lem:grad_norm_decomp}, the squared gradient norm can be written as
\begin{equation}
\label{eq:grad_norm_recall}
\|\nabla_{\bm{a}_{i,t}} \mathcal{J}(y_{i,t})\|^2
=
|w_{i,t}|^2
\left(
1 - 2\pi_\theta(y_{i,t})
+
\sum_{n=1}^{|\mathcal{V}|}\pi_\theta(v^n)^2
\right).
\end{equation}

To simplify the notation throughout this proof, we denote the components of the distribution $\pi_\theta \in \Delta_{|\mathcal{V}|-1}$ as $\pi_n \triangleq \pi_\theta(v^n)$ for $n \in \{1, \dots, |\mathcal{V}|\}$. 
The term $\sum_{n=1}^{|\mathcal{V}|}\pi_n^2$ represents the \emph{collision probability} of the distribution, which measures its concentration. We universally denote the Shannon entropy as $\mathcal{H}(\pi_\theta) = - \sum_{n=1}^{|\mathcal{V}|} \pi_n \ln \pi_n$.

\paragraph{Lower bound.}

To establish a rigorous lower bound, we relate the collision probability to the Shannon entropy. Recall that the Rényi entropy of order 2 is defined as $\mathcal{H}_2(\pi_\theta) = - \ln \left( \sum_{n=1}^{|\mathcal{V}|}\pi_n^2 \right)$.
Expressing the argument of the logarithm as an expectation, we have $\sum_{n=1}^{|\mathcal{V}|}\pi_n^2 = \mathbb{E}_{v \sim \pi_\theta}[\pi_\theta(v)]$. Because the function $f(x) = -\ln(x)$ is strictly convex, applying Jensen's inequality yields:
\[
\mathcal{H}_2(\pi_\theta)
=
- \ln \big( \mathbb{E}_{v \sim \pi_\theta}[\pi_\theta(v)] \big)
\le
\mathbb{E}_{v \sim \pi_\theta}[-\ln \pi_\theta(v)]
=
\mathcal H(\pi_\theta).
\]
Consequently, the collision probability satisfies $\sum_{n=1}^{|\mathcal{V}|}\pi_n^2 = e^{-\mathcal{H}_2(\pi_\theta)} \ge e^{-\mathcal H(\pi_\theta)}$. Substituting this inequality into Eq.~\eqref{eq:grad_norm_recall} yields the entropy-based lower bound on the gradient norm:
\begin{equation}
\|\nabla_{\bm{a}_{i,t}} \mathcal{J}(y_{i,t})\|^2
\ge
|w_{i,t}|^2
\left(
1 - 2\pi_\theta(y_{i,t})
+
e^{-\mathcal H(\pi_\theta)}
\right).
\end{equation}

\paragraph{Upper bound.}

We begin by formulating the objective via equivalent transformations. We introduce the Gini impurity $L = 1 - \sum_{n=1}^{|\mathcal{V}|} \pi_n^2$. The target inequality we aim to prove is:
\begin{equation}
\sum_{n=1}^{|\mathcal{V}|} \pi_n^2 \le 1 - C_V \mathcal{H}(\pi_\theta)^2,
\end{equation}
where $C_V = \frac{|\mathcal{V}|-1}{|\mathcal{V}|(\ln |\mathcal{V}|)^2}$. Rearranging the terms and applying the definition of $L$ yields the equivalent condition $C_V \mathcal{H}(\pi_\theta)^2 \le L$. At the vertices of the probability simplex (where a single component $\pi_n=1$ and all others vanish), $\mathcal{H}(\pi_\theta)=0$ and $L=0$, rendering the inequality trivially satisfied ($0 \le 0$). For any non-vertex distribution, $L > 0$, allowing us to reformulate the problem as bounding the functional $F(\pi_\theta) = \frac{\mathcal{H}(\pi_\theta)^2}{L} \le \frac{1}{C_V}$. Thus, it suffices to prove that the global maximum of $F(\pi_\theta)$ on the simplex $\Delta_{|\mathcal{V}|-1}$ satisfies $F(\pi_\theta) \le \frac{|\mathcal{V}|(\ln |\mathcal{V}|)^2}{|\mathcal{V}|-1}$.

To determine the interior constraints, let $\pi_\theta \in \Delta_{|\mathcal{V}|-1}$ denote a global maximizer of $F(\pi_\theta)$. Excluding the trivial vertices ensures $\mathcal{H}(\pi_\theta) > 0$ and $F(\pi_\theta) > 0$. By the Karush-Kuhn-Tucker (KKT) conditions \cite{boyd2004convex_ch5}, this extremum must reside in the relative interior of a sub-simplex characterized by $k$ non-zero components ($1 < k \le |\mathcal{V}|$). We apply the method of Lagrange multipliers to these non-zero components subject to the probability sum equality constraint. Because the logarithmic transformation is strictly monotonically increasing, it preserves the locations of extrema, allowing us to equivalently maximize $\Lambda = 2\ln \mathcal{H}(\pi_\theta) - \ln L - \lambda(\sum \pi_n - 1)$.
For any $\pi_i > 0$, the first-order stationarity condition is:
\begin{equation}
\frac{\partial}{\partial \pi_i}[2\ln \mathcal{H}(\pi_\theta) - \ln L] = \lambda \implies \frac{2}{\mathcal{H}(\pi_\theta)}(-1-\ln \pi_i) - \frac{1}{L}(-2\pi_i) = \lambda.
\end{equation}
Rearranging this yields $\frac{\pi_i}{L} - \frac{\ln \pi_i}{\mathcal{H}(\pi_\theta)} = \frac{\lambda}{2} + \frac{1}{\mathcal{H}(\pi_\theta)}$. Defining the auxiliary function $g(t) = \frac{t}{L} - \frac{\ln t}{\mathcal{H}(\pi_\theta)}$ for $t \in (0,1)$, its second derivative is $g''(t) = \frac{1}{\mathcal{H}(\pi_\theta) t^2}$. Since $\mathcal{H}(\pi_\theta) > 0$, $g''(t) > 0$ universally on this domain, confirming that $g(t)$ is strictly convex. Consequently, the equation $g(\pi_i) = C$ can possess at most two distinct roots, implying the non-zero components of the extremum can assume at most two distinct values.

We now prove that no two-value stationary points can exist. Assume, toward a contradiction, that the non-zero components take exactly two distinct values, $x$ and $y$ (with $x, y \in (0,1)$ and $x \neq y$), occurring with multiplicities $k'$ and $m'$, respectively. The constraints dictate $k'x + m'y = 1$ and $k'x^2 + m'y^2 = 1 - L$. Because both $x$ and $y$ satisfy the stationarity condition, $g(x) = g(y)$, which implies:
\begin{equation}
\frac{x}{L} - \frac{\ln x}{\mathcal{H}(\pi_\theta)} = \frac{y}{L} - \frac{\ln y}{\mathcal{H}(\pi_\theta)} \implies \mathcal{H}(\pi_\theta) = L \frac{\ln y - \ln x}{y - x}.
\end{equation}
Introducing the logarithmic difference quotient $\Delta = \frac{\ln y - \ln x}{y - x}$, we obtain $\mathcal{H}(\pi_\theta) = L\Delta$, which implies $\ln y = \ln x + \Delta(y-x)$. Substituting this into the entropy expression yields $\mathcal{H}(\pi_\theta) = -k'x\ln x - m'y(\ln x + \Delta(y-x))$. Grouping terms and applying $k'x + m'y = 1$ gives $\mathcal{H}(\pi_\theta) = -\ln x - m'y(y-x)\Delta$. Equating this with $L\Delta$ results in $-\ln x = [L + m'y(y-x)]\Delta$. Substituting $L = 1 - k'x^2 - m'y^2$ and simplifying the bracketed term reduces it to $1 - x(k'x + m'y) = 1 - x$. Consequently, $-\ln x = (1-x)\Delta$, yielding $\Delta = \frac{-\ln x}{1-x}$. By symmetry, we analogously obtain $\Delta = \frac{-\ln y}{1-y}$. Therefore, any two-value stationary point must satisfy:
\begin{equation}
\frac{-\ln x}{1-x} = \frac{-\ln y}{1-y}.
\end{equation}
Let $u(t) = \frac{-\ln t}{1-t}$ for $t \in (0,1)$. Its derivative is $u'(t) = \frac{1 - 1/t - \ln t}{(1-t)^2}$. Let $v(t) = 1 - \frac{1}{t} - \ln t$ denote the numerator. Since $v'(t) = \frac{1-t}{t^2} > 0$ for all $t \in (0,1)$, $v(t)$ is strictly monotonically increasing. Because $v(1) = 0$, it follows that $v(t) < 0$ on $(0,1)$. This establishes that $u'(t) < 0$ globally on this interval, implying $u(t)$ is strictly monotonically decreasing. Thus, $u(x) = u(y)$ holds if and only if $x = y$, directly contradicting $x \neq y$.

The preceding contradiction establishes that the stationary point must be a uniform distribution over a support of size $k$ ($1 < k \le |\mathcal{V}|$), where each of the $k$ non-zero components equals $\frac{1}{k}$. At this point, the functional evaluates to $h(k) = \frac{k(\ln k)^2}{k-1}$. Treating $k$ as a continuous variable on $(1, |\mathcal{V}|]$ and differentiating yields:
\begin{equation}
h'(k) = \frac{\ln k [2(k-1) - \ln k]}{(k-1)^2}.
\end{equation}
For $k > 1$, the standard inequality $\ln k < k - 1$ guarantees that $2(k-1) - \ln k > k - 1 > 0$. Because $\ln k > 0$, $h'(k) > 0$. This demonstrates that the objective function strictly monotonically increases with the support size $k$. Thus, the global maximum is uniquely attained at the full support $k=|\mathcal{V}|$:
\begin{equation}
\max_{\pi_\theta \in \Delta_{|\mathcal{V}|-1}} F(\pi_\theta) = h(|\mathcal{V}|) = \frac{|\mathcal{V}|(\ln |\mathcal{V}|)^2}{|\mathcal{V}|-1}.
\end{equation}

Given this global maximum, the inequality $\frac{\mathcal{H}(\pi_\theta)^2}{1 - \sum_{n=1}^{|\mathcal{V}|} \pi_n^2} \le \frac{|\mathcal{V}|(\ln |\mathcal{V}|)^2}{|\mathcal{V}|-1}$ holds universally for any distribution $\pi_\theta \in \Delta_{|\mathcal{V}|-1}$. Multiplying both sides by the non-negative denominator gives $\mathcal{H}(\pi_\theta)^2 \le \frac{|\mathcal{V}|(\ln |\mathcal{V}|)^2}{|\mathcal{V}|-1} \left(1 - \sum_{n=1}^{|\mathcal{V}|} \pi_n^2\right)$. Multiplying by the constant factor $\frac{|\mathcal{V}|-1}{|\mathcal{V}|(\ln |\mathcal{V}|)^2}$ and substituting the definition of $C_V$ results in $C_V \mathcal{H}(\pi_\theta)^2 \le 1 - \sum_{n=1}^{|\mathcal{V}|} \pi_n^2$. Isolating the summation term yields the desired upper bound:
\begin{equation}
\sum_{n=1}^{|\mathcal{V}|} \pi_n^2 \le 1 - C_V \mathcal{H}(\pi_\theta)^2.
\end{equation}

Combining these lower and upper bounds completes the proof of Theorem~\ref{prop:grad_bounds}.
\end{proof}

\section{Algorithm}\label{app:algorithm}
The complete STAPO procedure is summarized in Algorithm~\ref{alg:stapo}.
\floatstyle{ruled}
\restylefloat{algorithm}
\begin{algorithm}[!t]
\caption{Spurious-Token-Aware Policy Optimization (STAPO)}
\label{alg:stapo}
\begin{algorithmic}[1]
\Require Dataset $\mathcal{D}$, Initial Policy $\pi_\theta$, Group size $G$, Thresholds $\tau_p, \tau_h$, Batch size $B$
\State Initialize policy parameters $\theta$
\For{each training iteration}
    \State Synchronize: $\theta_{\mathrm{old}} \leftarrow \theta$ 
    \State Sample prompts $\bm{x}^B \sim \mathcal{D}$ and generate responses $\{\bm{y}_1, \dots, \bm{y}_G\}^B \sim \pi_{\theta_{\mathrm{old}}}(\cdot \mid \bm{x}^B)$
    \State Compute advantages $\hat{A}^B_i$ using Eq.~\eqref{eq:ratio_adv_def}
    \For{each mini-batch $ \subset \mathrm{Batch Data}$ of size $B$}
        \For{each response $\bm{y}_i$ and token $t$}
            \State Obtain $p_{i,t} = \pi_\theta(y_{i,t} \mid \bm{x}, \bm{y}_{i,<t})$ and $h_{i,t} = \mathcal{H}(\pi_\theta(\cdot \mid \bm{x}, \bm{y}_{i,<t}))$
            \State $\mathbb{I}^{\mathrm{S2T}}_{i,t} \leftarrow 1$ \Comment{Default: Keep}
            \State \textbf{if} $\hat{A}_i > 0 \land p_{i,t} < \tau_p \land h_{i,t} < \tau^{(q)}_h$ \textbf{then} $\mathbb{I}^{\mathrm{S2T}}_{i,t} \leftarrow 0$ \Comment{Spurious Tokens: Mask}
        \EndFor
        \State Update $\theta$ through Eq.~\eqref{eq:stapo_loss} with $\mathbb{I}^{\mathrm{S2T}}_{i,t}$
    \EndFor
\EndFor
\State \textbf{return} Final policy $\pi_\theta$
\end{algorithmic}
\end{algorithm}
\section{Training Details}
\label{app:traing_details}
We implement our proposed STAPO algorithm based on the open-source alignment framework \texttt{veRL}~\cite{sheng2025hybridflow}. We utilize the DAPO-Math-17K~\cite{yu2025dapo} as the training dataset, where each prompt is formatted with the instruction: ``Please reason step by step, and put your final answer within \textbackslash boxed\{\}''. All models are trained using the AdamW optimizer with a constant learning rate of $1 \times 10^{-6}$ and a warm-up phase of 10 steps. To ensure training stability, we apply a global gradient clipping norm of 1.0.

For efficient data generation, we utilize vLLM~\cite{kwon2023efficient} as the inference backend. The training process employs a global batch size of 256, with each prompt generating a group of $G=8$ rollouts. Following the DAPO formulation, we do not employ a separate value network or an additional KL divergence penalty term in the loss function; instead, we rely on group-relative advantages and the clipping mechanism to constrain policy updates. We conduct all experiments on 64 NVIDIA H20 GPUs, with each training session taking an average of 5 to 7 days. The complete set of hyperparameters used across all model scales is detailed in Table~\ref{tab:full_hyperparameters}.
\begin{table}[!t]
\centering
\caption{Full Training Hyperparameters} 
\label{tab:full_hyperparameters}
\begin{tabular}{lc}
\toprule
\textbf{Hyperparameter} & \textbf{Value} \\
\midrule
Advantage Estimator & GRPO \\
Reward Function & DAPO \\
Probability Threshold $\tau_p$ & 0.002 \\
Entropy Percentile $q$ & 75 \\
Use KL Loss & No \\
Use Entropy Regularization & No \\
Train Batch Size & 256 \\
Max Prompt Length & 1k \\
Max Response Length & 15k \\
PPO Mini Batch Size & 64 \\
Dynamic Micro Batch Size & Enabled \\
Clip Ratio Range & [0.8, 1.28] \\
Grad Clip & 1.0 \\
Learning Rate & 1e-6 \\
Warm-up Step & 10 \\
Training Temperature & 1.0 \\
Training Top-p & 1.0 \\
Validation Temperature & 1.0 or 0.7 \\
Validation Top-p & 1.0 or 0.9 \\
Rollout N & 8 \\
\bottomrule
\end{tabular}
\end{table}

\section{Supplementary Experimental Results}
\label{app:additional_experiments}

% ==========================================
% 小节 1：超参数敏感性消融
% ==========================================

% ==========================================
% 小节 2：大模型的训练动态
% ==========================================
\subsection{Supplementary Training Dynamics}
\label{app:additional_training_dynamics}

\begin{figure}[t!]
    \centering
    
    % 第二张大图：8B 规模
    \begin{subfigure}{\textwidth}
        \centering
        \includegraphics[width=1.0\linewidth]{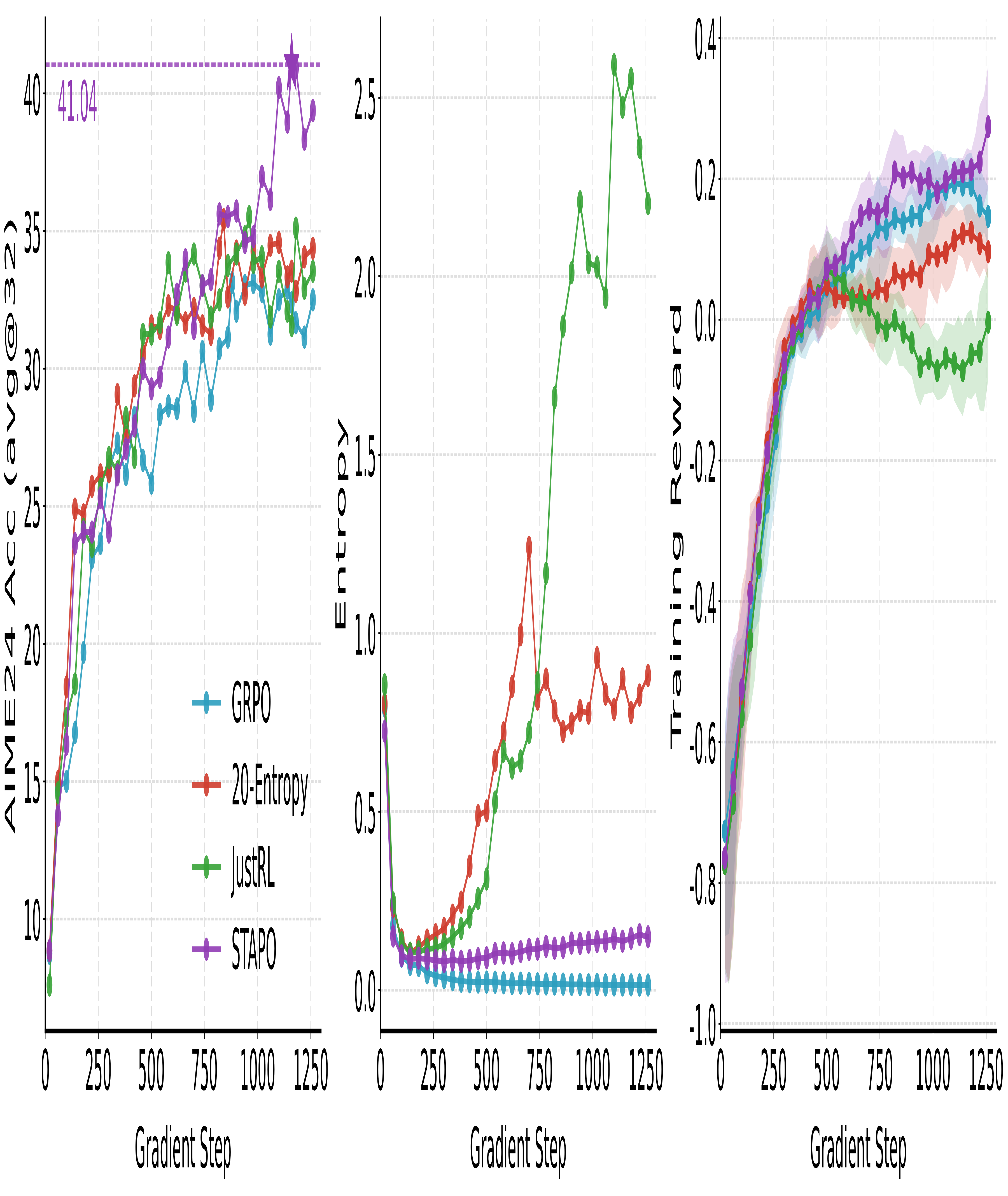}
        \caption{AIME24 Acc, entropy, and training reward for Qwen3-8B base.}
        \label{fig:app_8b}
    \end{subfigure}

    \vspace{10pt} % 附录空间相对宽裕，可以适当拉开一点间距

    % 第三张大图：14B 规模
    \begin{subfigure}{\textwidth}
        \centering
        \includegraphics[width=1.0\linewidth]{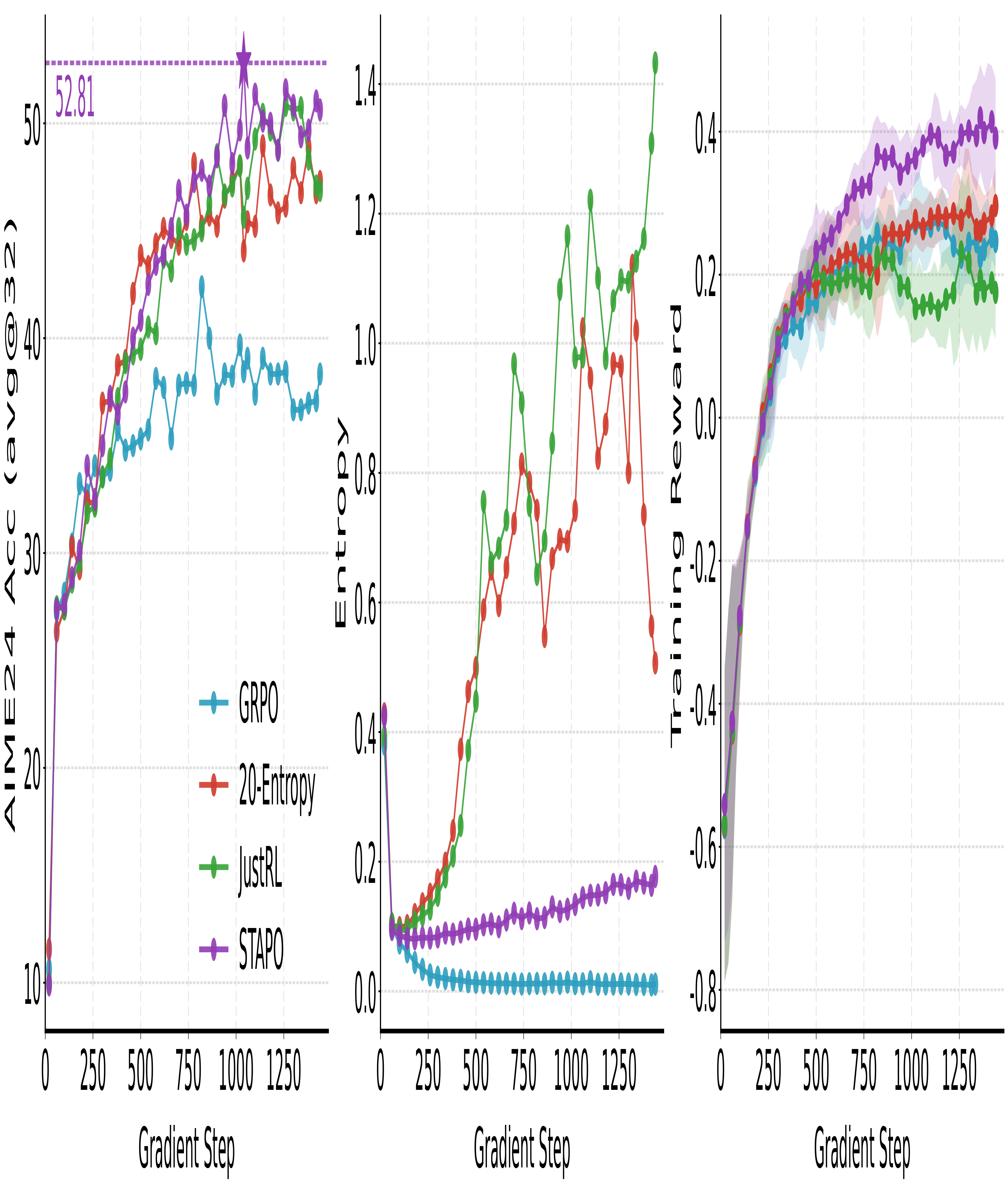}
        \caption{AIME24 Acc, entropy, and training reward for Qwen3-14B base.}
        \label{fig:app_14b}
    \end{subfigure}
    
    \caption{\textbf{Training Results on Larger Models.} The training dynamics for the Qwen3-8B and Qwen3-14B models. The results demonstrate that STAPO consistently maintains stable policy entropy and achieves higher accuracy across different parameter scales.}
    \label{fig:app_results}
\end{figure}

We further present the training dynamics for the larger model scales (Qwen3-8B and Qwen3-14B) in this section. As shown in Figure~\ref{fig:app_results}, the trends observed on the 1.7B model scale consistently to larger models. Across both 8B and 14B base models, STAPO successfully prevents premature entropy collapse, maintaining stable exploration throughout the training process and achieving state-of-the-art performance. To clearly visualize the extended training process, we downsampled the plotted curves, which were originally evaluated every 20 gradient steps, by retaining only every third data point. Crucially, the true peak performance of each accuracy curve remains strictly preserved.

% \subsection{Hyperparameter Sensitivity}
% \label{app:hyperparameter_sensitivity}
% \begin{figure}[t!]
%     \centering
%     \includegraphics[width=0.90\linewidth]{figures/tau_p_ablation.pdf}
%     \caption{\textbf{Sensitivity to Probability Threshold ($\tau_p$).} Training dynamics of AIME24 accuracy and policy entropy under different probability thresholds for the Qwen3-1.7B base model.}
%     \label{fig:tau_p}
% \end{figure}

% We evaluate the sensitivity of the Qwen3-1.7B base model to the probability threshold $\tau_p$. As illustrated in Figure~\ref{fig:tau_p}, model performance is highly dependent on the choice of $\tau_p$. Specifically on AIME24, as $\tau_p$ shifts from $2 \times 10^{-3}$ to $2 \times 10^{-2}$, accuracy drops sharply, accompanied by a collapse in policy entropy. Conversely, when $\tau_p$ is reduced to a more conservative $2 \times 10^{-4}$, accuracy initially improves but subsequently degrades, exhibiting highly volatile entropy dynamics. This trend indicates that while overly aggressive thresholds remove low-probability tokens crucial for maintaining valid reasoning chains, overly conservative filtering strategies fail to effectively eliminate spurious tokens. Therefore, effective filtering must remain highly selective, targeting only genuinely spurious tokens while preserving rare yet semantically meaningful ones.
\subsection{Analysis of Computational and Memory Overhead}
\label{app:stapo_overhead}

STAPO is a lightweight token-level optimization strategy
that introduces no additional model parameters, forward passes,
or backward passes. Its computational overhead is limited to
entropy computation, entropy-based token selection, mask
construction, and masked-objective application. To quantify
the associated runtime and memory overhead, we profile STAPO
under the same training configuration as JustRL, measuring
the additional runtime and GPU memory consumption per
micro-batch averaged over the first 100 training steps.
The results are summarized in Table~\ref{tab:stapo_overhead}.

\begin{table}[t]
    \centering
    \small
    \renewcommand{\arraystretch}{1.3}
    \setlength{\tabcolsep}{8pt}
    \caption{\textbf{Computational and Memory Overhead of STAPO.}
    Additional overhead per micro-batch, averaged over
    the first 100 training steps.
    Entropy: entropy computation;
    Selection: entropy-based token selection;
    Masking: mask construction and masked-objective application.}
    \label{tab:stapo_overhead}
    \begin{tabular}{lcc}
        \toprule
        \textbf{Operation}
        & \textbf{Time (ms)}
        & \textbf{Memory (MB)} \\
        \midrule
        Entropy
        & 11.939 & 0.474 \\
        Selection
        & 0.258 & 0.947 \\
        Masking
        & 0.452 & 2.487 \\
        \midrule
        \textbf{Total}
        & \textbf{12.650}
        & \textbf{3.908} \\
        \bottomrule
    \end{tabular}
\end{table}

As shown in Table~\ref{tab:stapo_overhead},
the total additional computational overhead is only
12.650 ms per micro-batch, with an additional memory
footprint of 3.908 MB.
Entropy computation accounts for most of the additional
runtime, while token selection and mask construction
introduce only minor costs.

For reference, the baseline training configuration
requires approximately 130 seconds per training update
and occupies approximately 60 GB of GPU memory.
The measured overhead is therefore negligible compared
with the overall training cost, with both the additional
runtime and memory footprint being approximately four
orders of magnitude smaller than the corresponding
baseline quantities.
These results demonstrate that STAPO can be integrated
into existing RL training pipelines with minimal
computational and memory overhead.

\subsection{Validation of Spurious Tokens}
\label{app:masked_token_validation}

To assess whether the tokens masked by STAPO are genuinely
spurious rather than legitimate low-probability tokens,
we conduct a token-level validation study on a randomly
sampled subset of 500 masked tokens.
Following the taxonomy in Appendix~\ref{app:spurious_examples}, we consider
three categories of spurious tokens:
\textit{Uncommon Syntax} (Category I),
\textit{Hallucination or Mathematical Error} (Category II),
and \textit{Formatting Error} (Category III).
We employ the GPT-5.5 API with a fixed evaluation rubric
to determine whether each masked token is spurious and,
if so, assign it to one of the three categories.
As an additional sanity check, we manually inspect
20 randomly selected cases, all of which agree with
the GPT-5.5 annotations.

Table~\ref{tab:masked_token_validation} summarizes
the validation results.
Among the 500 sampled masked tokens, GPT-5.5 identifies
474 (94.8\%) as spurious, including 269 cases of
uncommon syntax, 128 cases of hallucination or
mathematical errors, and 77 cases of formatting errors.
The remaining 26 tokens (5.2\%) are classified as
legitimate low-probability tokens.
These results suggest that STAPO achieves a high
estimated precision in identifying spurious tokens,
despite requiring neither an auxiliary evaluator
nor manual token annotations during training.

\begin{table}[t]
    \centering
    \small
    \renewcommand{\arraystretch}{1.3}
    \setlength{\tabcolsep}{8pt}
    \caption{Validation of 500 randomly sampled tokens
    masked by STAPO using GPT-5.5.}
    \label{tab:masked_token_validation}
    \begin{tabular}{lcc}
        \toprule
        \textbf{Category} & \textbf{Count} & \textbf{Ratio} \\
        \midrule
        Uncommon Syntax (I) & 269 & 53.8\% \\
        Hallucination / Math Error (II) & 128 & 25.6\% \\
        Formatting Error (III) & 77 & 15.4\% \\
        \midrule
        \textbf{Spurious (Total)} & \textbf{474} & \textbf{94.8\%} \\
        Legitimate (Non-spurious) & 26 & 5.2\% \\
        \midrule
        \textbf{Total} & \textbf{500} & \textbf{100\%} \\
        \bottomrule
    \end{tabular}
\end{table}

To better understand the remaining non-spurious cases,
we examine a representative example in which a token
is masked by STAPO but judged legitimate by GPT-5.5.
The model-generated reasoning proceeds as follows:

\begin{quote}
\emph{For a cuboid with side lengths $a$, $ar$, and
$ar^2$, the preceding derivation gives}
\[
a^3r^3 = 216
\;\Longrightarrow\;
(ar)^3 = 216
\;\Longrightarrow\;
ar = 6.
\]

\emph{The model then generates}
\[
\begin{aligned}
S &= 2(ab+bc+ca)\\
  &= 2\bigl(
      a\times ar
      + \boxed{a}\times ar^2
      + ar\times ar^2
     \bigr).
\end{aligned}
\]
\end{quote}

Here, the boxed token $\boxed{a}$ is masked
by STAPO. The top-5 token probabilities under the
current policy are:

\[
{ar}: 99.88\%, \quad
{a}: 0.12\%, \quad
{ab}: 0.00\%, \quad
\texttt{\textbackslash(}: 0.00\%, \quad
{b}: 0.00\%.
\]

Despite its low probability of 0.12\%, GPT-5.5
classifies the masked token $a$ as legitimate,
providing the following explanation:

\begin{quote}
\emph{``The target token $a$ begins the term
$a \times ar^2$, which is semantically correct as
$ca$, and is a natural continuation after
$a \times ar+$.''}
\end{quote}

This judgment is mathematically valid.
The token $a$ forms the product $a\times ar^2$,
corresponding to the $ca$ term in the standard
surface-area formula $2(ab+bc+ca)$.
Although the generated terms follow the order
$ab+ca+bc$ instead of the conventional order
$ab+bc+ca$, the two expressions are equivalent
due to the commutativity of addition.

This case illustrates that a small fraction of
STAPO-masked tokens can be semantically and
mathematically valid despite having low probability.
Such cases may arise from locally unexpected but
equivalent reasoning continuations, such as
reordering terms in an algebraic expression.
Nevertheless, the substantial probability discrepancy
indicates that these tokens deviate from the policy's
preferred local continuation, explaining why they
are selected by STAPO.

\subsection{Spurious Tokens in Mathematical Reasoning}
\label{app:spurious_examples}

We provide additional examples of spurious tokens
identified by STAPO in mathematical reasoning
trajectories. These examples are organized into
three categories: (I) uncommon syntax,
(II) hallucinations and mathematical errors, and
(III) formatting errors. For each case, we present
the relevant generation context, the masked token,
its probability and advantage, and the top-five
alternative tokens. The masked tokens are highlighted
in red.

% ================================================
% Spurious-token examples (mathematical reasoning)
% ================================================
% Load spurious_case_style_categorized.tex before this file.
% Keeps the original context text and all token statistics.
% Uses only prefixed stapoMath commands; does not modify existing codecase/masktok.
% Single-column categorized version: a distinct header and accent for each category.
% All 17 case contexts and all 20 token statistics are preserved.

\stapoMathCategory{I}{Uncommon Syntax}{stapoMathViolet}{5}
\label{app:uncommon_syntax_cases}

\stapoMathKeep{18\baselineskip}
\begin{stapoMathCase}{Case I.1: Unusual Word Choice}
\noindent\textbf{Context.}\enspace The length of $A$ is $1981$, and the length of $B$ is $1014$. Since the length of $U$ is $2012$, for the minimum length of $A \cap B$, $B$ should start immediately from the right of $A$. The length of $A \cap B$ is the sum of the lengths of $A$ and $B$ minus the length of $U$ , so its minimum value is $1981 + 1014 - 2012 = 983$. Here is a Python \stapoMathMask{calculation} to verify the above reasoning...
\stapoMathStat{\stapoMathMask{calculation}}{0.0115\%}{1.62}{code(51.90\%)\allowbreak\hspace{0.2em}\textcolor{black!40}{\textbullet}\hspace{0.2em} program(31.48\%)\allowbreak\hspace{0.2em}\textcolor{black!40}{\textbullet}\hspace{0.2em} solution(11.58\%)\allowbreak\hspace{0.2em}\textcolor{black!40}{\textbullet}\hspace{0.2em} script(3.32\%)\allowbreak\hspace{0.2em}\textcolor{black!40}{\textbullet}\hspace{0.2em} function(0.58\%)}
\end{stapoMathCase}

\stapoMathKeep{18\baselineskip}
\begin{stapoMathCase}{Case I.2: Inconsistent Terminology}
\noindent\textbf{Context.}\enspace We are given $n$ islands with ferry lines between each pair of distinct islands, which is essentially a complete graph $K_n$. Each edge of the graph is colored by one of $k$ colors, corresponding to the companies operating the ferry lines.

For any closure of one company, all its lines must be \stapoMathMask{broken} (the corresponding edges ignored) such that the remaining graph can no longer support a Hamiltonian Cycle, meaning it's no longer possible to visit all the islands exactly once without returning to the starting point.

According to the condition, every color, therefore, must serve to connect the graph in such a way that the removal of edges of any single color makes the graph Hamiltonian non-existent.

In a complete graph $K_n$, we can color the edges in such a way that removing any single color results in a graph that lacks a Hamiltonian Cycle. For this, we can use a result from graph theory, informally known as "edge coloring minimum spanning trees": for $n$ nodes, we can divide the set of nodes into $2\lfloor{\frac{n}{2}}\rfloor$ color classes (as far as coloring spans, and where $\lfloor{x}\rfloor$ means "the greatest integer less than $x$").

So when we have $n=2m$ (even) nodes, the maximal number of colors $k$ we can use is $m-1$, since we can divide our $n$ nodes into $m$ colors and deleting any one will still result in a complete graph $K_{2m-1}$ which remains Hamiltonian (since $K_{2m-1}$ is always Hamiltonian for $m\geq1$).

For $n=2m+1$ (odd) nodes, we can divide our $n$ nodes into $m$ classes, so the \stapoMathMask{discard} of any one of these $m$ classes will disintegrate the graph such that a Hamiltonian Cycle will become impossible because it contains a vertex of degree 1 or a subdivision thereof.
\stapoMathStat{\stapoMathMask{broken}}{0.0471\%}{2.47}{removed(85.53\%)\allowbreak\hspace{0.2em}\textcolor{black!40}{\textbullet}\hspace{0.2em} closed(3.32\%)\allowbreak\hspace{0.2em}\textcolor{black!40}{\textbullet}\hspace{0.2em} deleted(2.58\%)\allowbreak\hspace{0.2em}\textcolor{black!40}{\textbullet}\hspace{0.2em}  disconnected(1.57\%)\allowbreak\hspace{0.2em}\textcolor{black!40}{\textbullet}\hspace{0.2em}  taken(0.95\%)}
\stapoMathStat{\stapoMathMask{discard}}{0.0001\%}{2.47}{maximal(85.89\%)\allowbreak\hspace{0.2em}\textcolor{black!40}{\textbullet}\hspace{0.2em}  maximum(9.05\%)\allowbreak\hspace{0.2em}\textcolor{black!40}{\textbullet}\hspace{0.2em}  number(1.57\%)\allowbreak\hspace{0.2em}\textcolor{black!40}{\textbullet}\hspace{0.2em}  largest(0.45\%)\allowbreak\hspace{0.2em}\textcolor{black!40}{\textbullet}\hspace{0.2em}  answer(0.45\%)}
\end{stapoMathCase}

\stapoMathKeep{18\baselineskip}
\begin{stapoMathCase}{Case I.3: Redundant Answer Wording}
\noindent\textbf{Context.}\enspace Therefore, the final \stapoMathMask{boxed} answer is:\texttt{\detokenize{\[ \boxed{41} \]<|endoftext|>}}
\stapoMathStat{\stapoMathMask{boxed}}{0.1486\%}{0.54}{answer(98.88\%)\allowbreak\hspace{0.2em}\textcolor{black!40}{\textbullet}\hspace{0.2em}  result(0.52\%)\allowbreak\hspace{0.2em}\textcolor{black!40}{\textbullet}\hspace{0.2em}  value(0.19\%)\allowbreak\hspace{0.2em}\textcolor{black!40}{\textbullet}\hspace{0.2em}  boxed(0.15\%)\allowbreak\hspace{0.2em}\textcolor{black!40}{\textbullet}\hspace{0.2em}  solution(0.15\%)}
\end{stapoMathCase}

\stapoMathKeep{18\baselineskip}
\begin{stapoMathCase}{Case I.4: Awkward Preposition}
\noindent\textbf{Context.}\enspace To solve the given problem, we need to find the smallest number of candies $N$ that satisfies the given conditions. Specifically:

1. When $N$ candies are divided \stapoMathMask{between} 21 people (Albert and his 20 friends), the remainder is 5.
\stapoMathStat{\stapoMathMask{between}}{0.0667\%}{0.72}{among(64.55\%)\allowbreak\hspace{0.2em}\textcolor{black!40}{\textbullet}\hspace{0.2em}  by(23.75\%)\allowbreak\hspace{0.2em}\textcolor{black!40}{\textbullet}\hspace{0.2em}  evenly(8.74\%)\allowbreak\hspace{0.2em}\textcolor{black!40}{\textbullet}\hspace{0.2em}  amongst(1.95\%)\allowbreak\hspace{0.2em}\textcolor{black!40}{\textbullet}\hspace{0.2em}  equally(0.92\%)}
\end{stapoMathCase}

\stapoMathKeep{18\baselineskip}
\begin{stapoMathCase}{Case I.5: Corrupted Word Fragments}
\noindent\textbf{Context.}\enspace To \stapoMathMask{tr}isegment, we must first count the number of points in the polygon.

\texttt{\#\#\#} Step 1:sume the number of points in the polygon

Let $n$ \stapoMathMask{y}ea the number of sides (or vertices) in a polygon $P$ as given in the problem statement. We are given that $P_{1}, P_{2}, \ldots P_{n} /$ are the $n$ vertices of the polygon.
\stapoMathStat{\stapoMathMask{tr}}{0.0004\%}{0.72}{olve(63.83\%)\allowbreak\hspace{0.2em}\textcolor{black!40}{\textbullet}\hspace{0.2em}  determine(30.15\%)\allowbreak\hspace{0.2em}\textcolor{black!40}{\textbullet}\hspace{0.2em}  find(5.24\%)\allowbreak\hspace{0.2em}\textcolor{black!40}{\textbullet}\hspace{0.2em}  triang(0.20\%)\allowbreak\hspace{0.2em}\textcolor{black!40}{\textbullet}\hspace{0.2em}  count(0.10\%)}
\stapoMathStat{\stapoMathMask{y}}{0.0003\%}{0.72}{ be(88.02\%)\allowbreak\hspace{0.2em}\textcolor{black!40}{\textbullet}\hspace{0.2em}  represent(5.63\%)\allowbreak\hspace{0.2em}\textcolor{black!40}{\textbullet}\hspace{0.2em}  denote(3.41\%)\allowbreak\hspace{0.2em}\textcolor{black!40}{\textbullet}\hspace{0.2em}  =(0.76\%)\allowbreak\hspace{0.2em}\textcolor{black!40}{\textbullet}\hspace{0.2em}  and(0.59\%)}
\end{stapoMathCase}

\stapoMathCategory{II}{Hallucinations and Math Errors}{stapoMathRust}{6}
\label{app:hallucination_math_cases}

\stapoMathKeep{18\baselineskip}
\begin{stapoMathCase}{Case II.1: Incorrect Formula Index}
\noindent\textbf{Context.}\enspace $$A = \frac{1}{2}\left|\text{x}_1(\text{y}_2-\text{y}_3) + \text{x}_2(\text{y}_3-\text{y}_{\stapoMathMask{2}}) + \text{x}_3(\text{y}_1-\text{y}_2) \right|$$
Here, \$ x\stapoMathMask{\$s} and \$y\$s are the coordinates of the points.
\stapoMathStat{\stapoMathMask{2}}{0.0142\%}{0.54}{1(99.95\%)\allowbreak\hspace{0.2em}\textcolor{black!40}{\textbullet}\hspace{0.2em} 2(0.01\%)\allowbreak\hspace{0.2em}\textcolor{black!40}{\textbullet}\hspace{0.2em} 3(0.01\%)\allowbreak\hspace{0.2em}\textcolor{black!40}{\textbullet}\hspace{0.2em} 0(0.00\%)\allowbreak\hspace{0.2em}\textcolor{black!40}{\textbullet}\hspace{0.2em} \}\_(0.00\%)}
\stapoMathStat{\stapoMathMask{\$s}}{0.0117\%}{0.54}{\_(94.78\%)\allowbreak\hspace{0.2em}\textcolor{black!40}{\textbullet}\hspace{0.2em} 1(1.05\%)\allowbreak\hspace{0.2em}\textcolor{black!40}{\textbullet}\hspace{0.2em}  =(0.82\%)\allowbreak\hspace{0.2em}\textcolor{black!40}{\textbullet}\hspace{0.2em} \_\{(0.82\%)\allowbreak\hspace{0.2em}\textcolor{black!40}{\textbullet}\hspace{0.2em}  \$(0.50\%)}
\end{stapoMathCase}

\stapoMathKeep{18\baselineskip}
\begin{stapoMathCase}{Case II.2: Incorrect Numeric Value}
\noindent\textbf{Context.}\enspace \( (a - 10, b - 10) = (193, 11) \) → \( a = 203 \), \( b = 21\stapoMathMask{3} \)
\stapoMathStat{\stapoMathMask{3}}{0.0017\%}{0.94}{\texttt{\textbackslash} (99.98\%)\allowbreak\hspace{0.2em}\textcolor{black!40}{\textbullet}\hspace{0.2em} \texttt{\textbackslash n}(0.01\%)\allowbreak\hspace{0.2em}\textcolor{black!40}{\textbullet}\hspace{0.2em} 1(0.01\%)\allowbreak\hspace{0.2em}\textcolor{black!40}{\textbullet}\hspace{0.2em}  \texttt{\textbackslash\textbackslash n}(0.00\%)\allowbreak\hspace{0.2em}\textcolor{black!40}{\textbullet}\hspace{0.2em} 3(0.00\%)}
\end{stapoMathCase}

\stapoMathKeep{18\baselineskip}
\begin{stapoMathCase}{Case II.3: Incorrect Expression}
\noindent\textbf{Context.}\enspace $6901 = 67 \times 103 \stapoMathMask{-} 1$
\stapoMathStat{\stapoMathMask{-}}{0.0202\%}{1.21}{\$,(99.48\%)\allowbreak\hspace{0.2em}\textcolor{black!40}{\textbullet}\hspace{0.2em} \$(0.36\%)\allowbreak\hspace{0.2em}\textcolor{black!40}{\textbullet}\hspace{0.2em}  +(0.13\%)\allowbreak\hspace{0.2em}\textcolor{black!40}{\textbullet}\hspace{0.2em}  -(0.02\%)\allowbreak\hspace{0.2em}\textcolor{black!40}{\textbullet}\hspace{0.2em}  =(0.00\%)}
\end{stapoMathCase}

\stapoMathKeep{18\baselineskip}
\begin{stapoMathCase}{Case II.4: Incorrect Decimal Value}
\noindent\textbf{Context.}\enspace $-16a^3 = 350,$

$a^3 = -21.\stapoMathMask{5}625.$
\stapoMathStat{\stapoMathMask{5}}{0.2433\%}{1.62}{8(98.16\%)\allowbreak\hspace{0.2em}\textcolor{black!40}{\textbullet}\hspace{0.2em} 2(1.40\%)\allowbreak\hspace{0.2em}\textcolor{black!40}{\textbullet}\hspace{0.2em} 5(0.24\%)\allowbreak\hspace{0.2em}\textcolor{black!40}{\textbullet}\hspace{0.2em} 9(0.07\%)\allowbreak\hspace{0.2em}\textcolor{black!40}{\textbullet}\hspace{0.2em} 7(0.04\%)}
\end{stapoMathCase}

\stapoMathKeep{18\baselineskip}
\begin{stapoMathCase}{Case II.5: Incorrect Derived Value}
\noindent\textbf{Context.}\enspace Compute \( a \) for Each Pair:\\
$\displaystyle
\begin{aligned}
   (-5, 30) & \rightarrow a = -25 \\
   (-4, 12) & \rightarrow a = -8 \\
   (-3, 6) & \rightarrow a = -3 \\
   (-2, 3) & \rightarrow a = -1 \\
   (0, 0) & \rightarrow a = 0 \\
   (-7, -42) & \rightarrow a = 49 \\
   (-8, -24) & \rightarrow a = 32 \\
   (-9, -18) & \rightarrow a = \stapoMathMask{9} \\ 
   (-10, -15) & \rightarrow a = 5 \\
   (-12, -12) & \rightarrow a = 24
\end{aligned}
$
\stapoMathStat{\stapoMathMask{9}}{0.0123\%}{0.35}{2(99.98\%)\allowbreak\hspace{0.2em}\textcolor{black!40}{\textbullet}\hspace{0.2em} 9(0.01\%)\allowbreak\hspace{0.2em}\textcolor{black!40}{\textbullet}\hspace{0.2em} 1(0.01\%)\allowbreak\hspace{0.2em}\textcolor{black!40}{\textbullet}\hspace{0.2em} 3(0.00\%)\allowbreak\hspace{0.2em}\textcolor{black!40}{\textbullet}\hspace{0.2em} 8(0.00\%)}
\end{stapoMathCase}

\stapoMathKeep{18\baselineskip}
\begin{stapoMathCase}{Case II.6: Incorrect Fraction Denominator}
\noindent\textbf{Context.}\enspace Thus,
$$
\sin x = \frac{435}{\stapoMathMask{4}35} = \frac{435}{533}
$$
\stapoMathStat{\stapoMathMask{4}}{0.0867\%}{1.62}{5(95.28\%)\allowbreak\hspace{0.2em}\textcolor{black!40}{\textbullet}\hspace{0.2em} 3(3.69\%)\allowbreak\hspace{0.2em}\textcolor{black!40}{\textbullet}\hspace{0.2em} \texttt{\textbackslash}(0.64\%)\allowbreak\hspace{0.2em}\textcolor{black!40}{\textbullet}\hspace{0.2em} 6(0.11\%)\allowbreak\hspace{0.2em}\textcolor{black!40}{\textbullet}\hspace{0.2em} 4(0.09\%)}
\end{stapoMathCase}

\stapoMathCategory{III}{Formatting Errors}{stapoMathTeal}{6}
\label{app:formatting_error_cases}

\stapoMathKeep{18\baselineskip}
\begin{stapoMathCase}{Case III.1: Unexpected Step Marker}
\noindent\textbf{Context.}\enspace To solve this problem, let's start by using the given information about the function \( f(x) = ax^2 + bx + c \):

1. Since \( f(1) = 0 \), we know that:
   $$
   a + b + c = 0
   $$
   This means \( c = -a - b \).

\stapoMathMask{Now}, substitute \( c = -a - b \) into the quadratic function:
   $$
   f(x) = ax^2 + bx - (a + b)
   $$
\stapoMathStat{\stapoMathMask{Now}}{0.0552\%}{0.35}{2(99.74\%)\allowbreak\hspace{0.2em}\textcolor{black!40}{\textbullet}\hspace{0.2em} Next(0.12\%)\allowbreak\hspace{0.2em}\textcolor{black!40}{\textbullet}\hspace{0.2em} Now(0.06\%)\allowbreak\hspace{0.2em}\textcolor{black!40}{\textbullet}\hspace{0.2em} So(0.03\%)\allowbreak\hspace{0.2em}\textcolor{black!40}{\textbullet}\hspace{0.2em} Sub(0.02\%)}
\end{stapoMathCase}

\stapoMathKeep{18\baselineskip}
\begin{stapoMathCase}{Case III.2: Malformed LaTeX Output}
\noindent\textbf{Context.}\enspace \texttt{\textbackslash}boxed\{5\stapoMathMask{\textbackslash}text\{ agony\}\}
\stapoMathStat{\stapoMathMask{\textbackslash}}{0.0015\%}{0.35}{\texttt{\}\textbackslash n}(99.57\%)\allowbreak\hspace{0.2em}\textcolor{black!40}{\textbullet}\hspace{0.2em} \}(0.25\%)\allowbreak\hspace{0.2em}\textcolor{black!40}{\textbullet}\hspace{0.2em} \}\texttt{\textbackslash n\textbackslash n}(0.12\%)\allowbreak\hspace{0.2em}\textcolor{black!40}{\textbullet}\hspace{0.2em} \}\texttt{\textbackslash}(0.03\%)\allowbreak\hspace{0.2em}\textcolor{black!40}{\textbullet}\hspace{0.2em} \}.\texttt{\textbackslash n}(0.01\%)}
\end{stapoMathCase}

\stapoMathKeep{18\baselineskip}
\begin{stapoMathCase}{Case III.3: Broken Math Delimiters}
\noindent\textbf{Context.}\enspace \texttt{\textbackslash therefore 10\textasciicircum y}\stapoMathMask{\texttt{)\$}}\texttt{ \textbackslash(\textbackslash sqrt\{25\}=5}
\stapoMathStat{\stapoMathMask{\texttt{)\$}}}{0.0063\%}{0.72}{=(84.69\%)\allowbreak\hspace{0.2em}\textcolor{black!40}{\textbullet}\hspace{0.2em} =\texttt{\textbackslash}(6.95\%)\allowbreak\hspace{0.2em}\textcolor{black!40}{\textbullet}\hspace{0.2em} =(4.22\%)\allowbreak\hspace{0.2em}\textcolor{black!40}{\textbullet}\hspace{0.2em} \texttt{\textvisiblespace\textbackslash}n(1.55\%)\allowbreak\hspace{0.2em}\textcolor{black!40}{\textbullet}\hspace{0.2em} \texttt{\textbackslash}n(0.35\%)}
\end{stapoMathCase}

\stapoMathKeep{18\baselineskip}
\begin{stapoMathCase}{Case III.4: Missing Mathematical Spacing}
\noindent\textbf{Context.}\enspace \[
\begin{aligned}
\Sigma_3 ={}& 11^3 + 10^3 + 9^3 + 8^3 + 7^3 \\
           & + 6^3 + 5^3 + 4^3 +\!\stapoMathMask{3}^3 \\
           & + 2^3 + 1^3 + 0^3
\end{aligned}
\]
\stapoMathStat{\stapoMathMask{3}}{0.0710\%}{1.62}{\texttt{\textvisiblespace}(99.92\%)\allowbreak\hspace{0.2em}\textcolor{black!40}{\textbullet}\hspace{0.2em} 3(0.07\%)\allowbreak\hspace{0.2em}\textcolor{black!40}{\textbullet}\hspace{0.2em}  \texttt{\textbackslash}n(0.00\%)\allowbreak\hspace{0.2em}\textcolor{black!40}{\textbullet}\hspace{0.2em}  \texttt{\textbackslash}(0.00\%)\allowbreak\hspace{0.2em}\textcolor{black!40}{\textbullet}\hspace{0.2em}  +(0.00\%)}
\end{stapoMathCase}

\stapoMathKeep{18\baselineskip}
\begin{stapoMathCase}{Case III.5: Extraneous Punctuation}
\noindent\textbf{Context.}\enspace Repeating these steps three more times for each roll back to $P_4$:
\begin{itemize}
  \setlength{\itemsep}{0pt}
  \setlength{\topsep}{2pt}
  \item If $A$ was rolled, $P_5 = (56, 368)$
  \item If $B$ was rolled, $P_5 = (56, 184)$
  \item If $C$ was rolled, $P_5 = (768, 184)$\stapoMathMask{:}
\end{itemize}
\stapoMathStat{\stapoMathMask{:}}{0.0035\%}{2.47}{\texttt{\textbackslash}n\texttt{\textbackslash}n(88.06\%)\allowbreak\hspace{0.2em}\textcolor{black!40}{\textbullet}\hspace{0.2em}  ((4.38\%)\allowbreak\hspace{0.2em}\textcolor{black!40}{\textbullet}\hspace{0.2em} \texttt{\textbackslash}n(4.38\%)\allowbreak\hspace{0.2em}\textcolor{black!40}{\textbullet}\hspace{0.2em}  \texttt{\textvisiblespace\textbackslash}n\texttt{\textbackslash}n(1.26\%)\allowbreak\hspace{0.2em}\textcolor{black!40}{\textbullet}\hspace{0.2em} ,(0.36\%)}
\end{stapoMathCase}

\stapoMathKeep{18\baselineskip}
\begin{stapoMathCase}{Case III.6: Corrupted Output Suffix}
\noindent\textbf{Context.}\enspace Hence, the maximum value of \texttt{\$\textbackslash frac\{n\_i\}\{k\}\$} for \texttt{1 \textbackslash leq i \textbackslash leq 70} is: \newline
\texttt{\textbackslash[} \newline
\texttt{\textbackslash boxed\{553\}} \newline
\texttt{\textbackslash]} \newline
\texttt{\stapoMathMask{struck}user<|endoftext|>}
\stapoMathStat{\stapoMathMask{struck}}{0.0000\%}{1.21}{\texttt{\textbackslash}(98.58\%)\allowbreak\hspace{0.2em}\textcolor{black!40}{\textbullet}\hspace{0.2em} \texttt{\textbackslash})(0.13\%)\allowbreak\hspace{0.2em}\textcolor{black!40}{\textbullet}\hspace{0.2em}  \texttt{\textbackslash}(0.10\%)\allowbreak\hspace{0.2em}\textcolor{black!40}{\textbullet}\hspace{0.2em} \$(0.08\%)\allowbreak\hspace{0.2em}\textcolor{black!40}{\textbullet}\hspace{0.2em} \$\$(0.07\%)}
\end{stapoMathCase}

These cases illustrate that even trajectories
with correct final answers may contain locally
undesirable tokens, ranging from unusual word choices
to incorrect intermediate calculations and formatting
artifacts. Such tokens may still be reinforced by
response-level rewards, highlighting the need for
finer-grained token-level optimization.

\subsection{Spurious Tokens in Code Reasoning}
\label{app:code_spurious_cases}

We further examine spurious tokens identified by
STAPO in code reasoning trajectories. Similar to
mathematical reasoning, positively rewarded coding
responses may contain locally unreliable or
unnecessary tokens. We present three representative
cases involving flawed code, invalid function calls,
and formatting artifacts. For each case, we show
the relevant code fragment and explain why the
masked token is considered spurious and how the
response may still receive a positive reward.
The masked tokens are highlighted in red.

\begin{codecase}{Case 1: Flawed Code}

\begin{lstlisting}[style=maskedcode]
visited = [0 for _ in range(N + 1)]
(*@\masktok{i}@*) = 0
while any(discovered):
    ...
\end{lstlisting}

\begin{casenote}{reasoncolor}{Spurious Reason}
The masked token occurs within a flawed graph-search
procedure, where the surrounding loop logic is
unreliable and inconsistent with the intended
component-search structure.
\end{casenote}

\begin{casenote}{rewardcolor}{Why a Positive Reward Is Possible}
The generated solution may still pass the evaluated
test cases if the problematic execution path is
not triggered.
\end{casenote}

\end{codecase}

\begin{codecase}{Case 2: Invalid Python Syntax}

\begin{lstlisting}[style=maskedcode]
if a < 3200:
    (*@\masktok{println}@*)("red")
\end{lstlisting}

\begin{casenote}{reasoncolor}{Spurious Reason}
The masked token \texttt{println} is not a valid
Python built-in function; the intended function
is \texttt{print}.
\end{casenote}

\begin{casenote}{rewardcolor}{Why a Positive Reward Is Possible}
The response also contains a valid Python solution
block. The erroneous token appears in an intermediate
or redundant fragment that is not used for evaluation.
\end{casenote}

\end{codecase}

\begin{codecase}{Case 3: Formatting Artifact}

\begin{lstlisting}[style=maskedcode]
return True
(*@\masktok{\textbackslash n\textbackslash n\textbackslash n}@*)
maxNumCaptures = 0
\end{lstlisting}

\begin{casenote}{reasoncolor}{Spurious Reason}
The masked sequence consists of excessive newline
characters, introducing unnecessary formatting
between code fragments without contributing to
the solution logic.
\end{casenote}

\begin{casenote}{rewardcolor}{Why a Positive Reward Is Possible}
The formatting artifact does not affect the final
extracted executable code and therefore does not
change the evaluation outcome.
\end{casenote}

\end{codecase}

These examples illustrate that positive
response-level rewards do not necessarily imply
the correctness or relevance of every generated token in code reasoning.
In particular, coding responses may contain
unreliable intermediate logic, invalid syntax in
unevaluated fragments, or redundant formatting,
while still producing an executable solution that
passes the evaluation.
By identifying and suppressing such locally
spurious tokens, STAPO provides a finer-grained
optimization signal beyond response-level rewards.

\end{document}